\documentclass{article}

\usepackage{microtype}
\usepackage{graphicx}
\usepackage{subfigure}
\usepackage{booktabs} 

\usepackage{hyperref}

\usepackage[accepted]{icml2023}

\usepackage{amsmath}
\usepackage{amssymb}
\usepackage{mathtools}
\usepackage{amsthm}

\usepackage[capitalize,noabbrev]{cleveref}

\theoremstyle{plain}
\newtheorem{theorem}{Theorem}[section]

\newtheorem{lemma}[theorem]{Lemma}
\newtheorem{corollary}[theorem]{Corollary}
\theoremstyle{definition}

\theoremstyle{remark}
\newtheorem{remark}[theorem]{Remark}
\usepackage{color}
\usepackage{bm}
\usepackage{hyperref}
\usepackage{amsmath,amsfonts}
\usepackage{blindtext}
\usepackage{listings}
\usepackage{mathtools}
\usepackage[normalem]{ulem}
\usepackage{multicol}
\usepackage{multirow}

\newcommand{\SM}{\text{SM}}
\newcommand{\FM}{\text{FM}}
\newcommand{\NM}{\text{NM}}
\newcommand{\DM}{\text{DM}}
\newcommand{\CFM}{\text{CFM}}

\newcommand{\ODE}{\text{ODE}}

\DeclarePairedDelimiterX{\infdivx}[2]{(}{)}{%
	#1\;\delimsize\|\;#2%
}

\newcommand{\kl}{D_{\mathrm{KL}}\infdivx}

\newcommand{\vect}[1]{\bm{#1}}

\newcommand{\argmin}{\operatornamewithlimits{argmin}}

\newcommand{\x}{\xv}

\newcommand{\s}{\sv}

\newcommand{\dm}{\mathrm{d}}

\newcommand{\E}{\mathbb{E}}
\newcommand{\R}{\mathbb{R}}
\newcommand{\N}{\mathcal{N}}

\newcommand{\dxv}{\mathrm{d}\xv}

\newcommand{\dt}{\mathrm{d}t}

\newcommand{\epsilonv}{\vect\epsilon}

\newcommand{\muv}{\vect\mu}

\newcommand{\phiv}{\vect\phi}

\newcommand{\ellv}{\vect \ell}

\newcommand{\sv}{\vect s}

\newcommand{\uv}{\vect u}
\newcommand{\vv}{\vect v}
\newcommand{\wv}{\vect w}
\newcommand{\xv}{\vect x}

\newcommand{\Iv}{\vect I}

\newcommand{\Xv}{\vect X}

\newcommand{\Hc}{\mathcal H}

\newcommand{\Jc}{\mathcal J}

\newcommand{\Lc}{\mathcal L}

\newcommand{\Nc}{\mathcal N}

\newcommand{\Tc}{\mathcal T}
\newcommand{\Uc}{\mathcal U}

\newcommand{\Var}{\mbox{Var}}

\newcommand{\tr}{\mathrm{tr}}

\newcommand\Tstrut{\rule{0pt}{2.6ex}}         
\newcommand\Bstrut{\rule[-0.9ex]{0pt}{0pt}}   

\usepackage[textsize=tiny]{todonotes}

\icmltitlerunning{Improved Techniques for Maximum Likelihood Estimation for Diffusion ODEs}

\begin{document}

\twocolumn[
\icmltitle{Improved Techniques for Maximum Likelihood Estimation for Diffusion ODEs}



\icmlsetsymbol{equal}{*}

\begin{icmlauthorlist}

\icmlauthor{Kaiwen Zheng}{equal,thu}
\icmlauthor{Cheng Lu}{equal,thu}
\icmlauthor{Jianfei Chen}{thu}
\icmlauthor{Jun Zhu}{thu,pazhou}

\end{icmlauthorlist}

\icmlaffiliation{thu}{Dept. of Comp. Sci. \& Tech., Institute for AI, BNRist Center, Tsinghua-Bosch Joint ML Center, THBI Lab, Tsinghua University}
\icmlaffiliation{pazhou}{Pazhou Lab (Huangpu), Guangzhou, China}

\icmlcorrespondingauthor{Jun Zhu}{dcszj@tsinghua.edu.cn}

\icmlkeywords{Machine Learning, ICML}

\vskip 0.3in
]



\printAffiliationsAndNotice{\icmlEqualContribution} 

\begin{abstract}
Diffusion models have exhibited excellent performance in various domains. The probability flow ordinary differential equation (ODE) of diffusion models (i.e., diffusion ODEs) is a particular case of continuous normalizing flows (CNFs), which enables deterministic inference and exact likelihood evaluation. However, the likelihood estimation results by diffusion ODEs are still far from those of the state-of-the-art likelihood-based generative models. In this work, we propose several improved techniques for maximum likelihood estimation for diffusion ODEs, including both training and evaluation perspectives. For training, we propose velocity parameterization and explore variance reduction techniques for faster convergence. We also derive an error-bounded high-order flow matching objective for finetuning, which improves the ODE likelihood and smooths its trajectory. For evaluation, we propose a novel training-free truncated-normal dequantization to fill the training-evaluation gap commonly existing in diffusion ODEs. Building upon these techniques, we achieve state-of-the-art likelihood estimation results on image datasets (2.56 on CIFAR-10, 3.43/3.69 on ImageNet-32) without variational dequantization or data augmentation, and 2.42 on CIFAR-10 with data augmentation. Code is available at \url{https://github.com/thu-ml/i-DODE}.
\end{abstract}

\section{Introduction}
Likelihood is an important metric to evaluate density estimation models, and accurate likelihood estimation is the key for many applications such as data compression~\citep{ho2021anfic,helminger2020lossy,kingma2021variational,yang2022lossy}, anomaly detection~\citep{chen2018autoencoder,dias2020anomaly} and out-of-distribution detection~\citep{serra2019input,xiao2020likelihood}. Many deep generative models can compute tractable likelihood, including autoregressive models~\citep{van2016conditional,salimans2017pixelcnn++,chen2018pixelsnail}, variational auto-encoders (VAE)~\citep{kingma2013auto,vahdat2020nvae}, normalizing flows ~\citep{dinh2016density,kingma2018glow,ho2019flow++} and diffusion models~\citep{sohl2015deep,song2019generative,ho2020denoising,song2020score,song2020denoising,karras2022elucidating}. Among these models, recent work named variational diffusion models (VDM)~\citep{kingma2021variational} achieves state-of-the-art likelihood estimation performance on standard image density estimation benchmarks, which is a variant of diffusion models.

There are two types of diffusion models, one is based on the reverse stochastic differential equation (SDE)~\citep{song2020score}, named as \textit{diffusion SDE}; the other is based on the probability flow ordinary differential equation (ODE)~\citep{song2020score}, named as \textit{diffusion ODE}. These two types of diffusion models define and evaluate the likelihood in different manners: diffusion SDE can be understood as an infinitely-deep VAE~\citep{huang2021variational} and can only compute a variational lower bound of the likelihood~\citep{song2020score, kingma2021variational}; while diffusion ODE is a variant of continuous normalizing flows~\citep{chen2018neural} and can compute the exact likelihood by ODE solvers.
Thus, it is natural to hypothesize that the likelihood performance of diffusion ODEs may be better than that of diffusion SDEs. 
However, all existing methods for training diffusion ODEs~\citep{song2021maximum,lu2022maximum,lipman2022flow,albergo2022building,liu2022flow} cannot even achieve a comparable likelihood performance with VDM, which belongs to diffusion SDEs. It still remains largely open whether diffusion ODEs are also great likelihood estimators.

Real-world data is usually discrete, and evaluating the likelihood of discrete data by diffusion ODEs needs to first perform a dequantization process~\citep{dinh2016density,salimans2017pixelcnn++} to make sure the input data of diffusion ODEs is continuous. In this work, we observe that previous likelihood evaluation of diffusion ODEs has flaws in the dequantization process: the uniform dequantization~\citep{song2021maximum} causes a large training-evaluation gap, and the variational dequantization~\citep{ho2019flow++,song2021maximum} requires additional training overhead and is hard to train to the optimal.

In this work, we propose several improved techniques, including both the evaluation perspective and training perspective, to allow the likelihood estimation by diffusion ODEs to outperform the existing state-of-the-art likelihood estimators. In the aspect of evaluation, we propose a training-free dequantization method dedicated to diffusion models by a carefully designed truncated-normal distribution, which can fit diffusion ODEs well and improve the likelihood evaluation by a large margin compared to uniform dequantization. We also introduce an importance-weighted likelihood estimator to get a tighter bound. In the aspect of training, we split our training into pretraining and finetuning phases. For pretraining, we propose a new model parameterization method including velocity parameterization, which is an extended version of flow matching~\citep{lipman2022flow} with practical modifications, and log-signal-to-noise-ratio timed parameterization. Besides, we find a simple yet efficient importance sampling strategy for variance reduction. Together, our pretraining has a faster convergence speed compared to previous work. For finetuning, we propose an error-bounded high-order flow matching objective, which not only improves the ODE likelihood but also results in smoother trajectories. Together, we name our framework Improved Diffusion ODE (i-DODE).

We conduct ablation studies to demonstrate the effectiveness of separate parts. Our experimental results empirically achieve the state-of-the-art likelihood on image datasets (2.56 on CIFAR-10, 3.43/3.69 on ImageNet-32), surpassing the previous best ODEs of 2.90 and 3.48/3.82, with the superiority that we use no data augmentation and throw away the need for training variational dequantization models. 

\section{Diffusion Models}
\subsection{Diffusion ODEs and Maximum Likelihood Training}
Suppose we have a $d$-dimensional data distribution $q_0(\x_0)$. Diffusion models~\citep{ho2020denoising,song2020score} gradually diffuse the data by a forward stochastic differential equation (SDE) starting from $\x_0\sim q_0(\x_0)$:
\begin{equation}
\label{eqn:forward_sde}
    \dm\x_t=f(t)\x_t\dt+g(t)\dm \wv_t,\quad \x_0\sim q_0(\x_0),
\end{equation}
where $f(t),g(t)\in\R$ are manually designed noise schedules and $\wv_t\in\R^d$ is a standard Wiener process. The forward process $\{\x_t\}_{t\in [0,T]}$ is accompanied with a series of marginal distributions $\{q_t\}_{t\in [0,T]}$, so that $q_T(\x_T)\approx \N(\x_T|\vect{0},\sigma_T^2\Iv)$ with some constant $\sigma_T>0$. Since this is a simple linear SDE, the transition kernel is an analytical Gaussian~\citep{song2020score}: $q_{0t}(\x_t|\x_0)=\mathcal{N}(\alpha_t\x_0,\sigma_t^2\Iv)$, where the coefficients satisfy $f(t)=\frac{\dm\log \alpha_t}{\dm t}$, $g^2(t)=\frac{\dm \sigma_t^2}{\dm t}-2\frac{\dm\log \alpha_t}{\dm t}\sigma_t^2$~\citep{kingma2021variational}.
Under some regularity conditions~\citep{anderson1982reverse}, the forward process has an
equivalent probability flow ODE~\citep{song2020score}:
\begin{equation}
\label{eqn:probability_flow_ODE}
    \frac{\dxv_t}{\dt}=f(t)\x_t-\frac{1}{2}g^2(t)\nabla_{\x}\log q_t(\x_t),
\end{equation}
which produces the same marginal distribution $q_t$ at each time $t$ as that in Eqn.~\eqref{eqn:forward_sde}. The only unknown term $\nabla_{\x}\log q_t(\x_t)$ is the \textit{score function} of $q_t$.
By parameterizing a \textit{score network} $\s_\theta(\x_t,t)$ to predict the time-dependent $\nabla_{\x}\log q_t(\x_t)$, we can replace the true score function, resulting in the \textit{diffusion ODE}~\citep{song2020score}:
\begin{equation}
\label{eqn:diff_ode_raw}
    \frac{\dxv_t}{\dt}=f(t)\x_t-\frac{1}{2}g^2(t)\s_\theta(\x_t,t),
\end{equation}
with the associated marginal distributions $\{p_t\}_{t\in [0,T]}$. Diffusion ODEs are special cases of \textit{continuous normalizing flows} (CNFs)~\citep{chen2018neural}, thus can perform exact inference of the latents and exact likelihood evaluation.

Though traditional maximum likelihood training methods for CNFs~\citep{grathwohl2018ffjord} are feasible for diffusion ODEs, the training costs of these methods are quite expensive and hard to scale up because of the requirement of solving ODEs at each iteration.
Instead, a more practical way is to match the generative probability flow $\{p_t\}_{t\in [0,T]}$ with $\{q_t\}_{t\in [0,T]}$ by a simulation-free approach. 
Specifically, \citet{lu2022maximum} proves that $\kl{q_0}{p_0^{\ODE}}$ can be formulated by $\kl{q_0}{p^{\ODE}_0}\!=\!\kl{q_T}{p^{\ODE}_T}+\Jc_{\ODE}(\theta)$, where

\begin{align}
\Jc_{\ODE}(\theta)\!&\coloneqq\!\frac{1}{2}\!\int_0^T \!\!\!g(t)^2\E_{q_t(\x_t)}\!\Big[(\s_\theta(\x_t,t)\!-\!\nabla_{\x}\log q_t(\x_t))^\top \nonumber\\ 
&(\nabla_{\x}\log p_t(\x_t)\!-\!\nabla_{\x}\log q_t(\x_t))\Big]\dt \label{eqn:J_ODE}
\end{align}
However, computing $\nabla_{\x}\log p_t(\x_t)$ requires solving another ODE and is also expensive~\citep{lu2022maximum}. To minimize $\Jc_{\ODE}(\theta)$ in a simulation-free manner, \citet{lu2022maximum} also proposes a combination of $g^2(t)$ weighted first-order and high-order score matching objectives. Particularly, the first-order score matching objective is
\begin{equation}
\label{eqn:sm_raw}
   \Jc_{\SM}(\theta) \coloneqq \int_0^T\frac{g^2(t)}{2\sigma_t^2}\E_{\x_0,\epsilonv}\left[\|\sigma_t\s_\theta(\x_t,t)+\epsilonv\|_2^2\right]\dt,
\end{equation}
where $\x_t = \alpha_t\x_0 + \sigma_t\epsilonv$, $\x_0 \sim q_0(\x_0)$ and $\epsilonv\sim \N(\epsilonv|\vect{0},\Iv)$.

\subsection{Log-SNR Timed Diffusion Models}
Diffusion models have manually designed noise schedule $\alpha_t,\sigma_t$, which has high freedom and affects the performance. Even for restricted design space such as Variance Preserving (VP)~\citep{song2020score}, which constrains the noise schedule by $\alpha_t^2+\sigma_t^2=1$, we could still have various choices about how fast $\alpha_t,\sigma_t$ changes w.r.t time $t$. To decouple the specific schedule form, variational diffusion models (VDM)~\citep{kingma2021variational} use a negative log-signal-to-noise-ratio (log-SNR) for the time variable and can greatly simplify both noise schedules and training objectives. Specifically, denote $\gamma_t=-\text{log-SNR}(t)=-\log\frac{\alpha_t^2}{\sigma_t^2}$, the change-of-variable relation from $\gamma$ to $t$ is
\begin{equation}
    \frac{\dm\gamma}{\dm t}=\frac{g^2(t)}{\sigma_t^2},
\end{equation}
and replace the time subscript with $\gamma$, we get the simplified score matching objective with likelihood weighting:
\begin{equation}
\label{eqn:J_SM_logSNR}
    \Jc_{\SM}(\theta)
    = \frac{1}{2}\int_{\gamma_0}^{\gamma_T}\E_{\x_0,\epsilonv}\left[\|\sigma_t\s_\theta(\x_\gamma,\gamma)+\epsilonv\|_2^2\right]\dm\gamma
\end{equation}
This result is in accordance with the continuous diffusion loss in \citet{kingma2021variational}.

\subsection{Dequantization for Density Estimation}
Many real-world datasets usually contain discrete data, such as images or texts. In such cases, learning a continuous density model to these discrete data points will cause degenerate results~\citep{uria2013rnade} and cannot provide meaningful density estimations. A common solution is \textit{dequantization}~\citep{dinh2016density,salimans2017pixelcnn++,ho2019flow++}. Specifically, suppose $\x_0$ is 8-bit discrete data scaled to $[-1,1]$. Dequantization methods assume that we have trained a continuous model distribution $p_{\text{model}}$ for $\x_0$, and define the discrete model distribution by
\begin{equation*}
    P_{\text{model}}(\x_0) \coloneqq \int_{[-\frac{1}{256}, \frac{1}{256})^d} p_{\text{model}}(\x_0+\uv)\dm\uv.
\end{equation*}
To train $P_{\text{model}}(\x_0)$ by maximum likelihood estimation, variational dequantization~\citep{ho2019flow++} introduces a dequantization distribution $q(\uv|\x_0)$ and jointly train $p_{\text{model}}$ and $q(\uv|\x_0)$ by a variational lower bound:
\begin{equation*}
\log\! P_{\text{model}}(\x_0)
\!\geq\!
\E_{q(\uv|\x_0)}\!\left[
    \log p_{\text{model}}(\x_0 \!+\! \uv)
    \!-\! \log q(\uv|\x_0)
\right].
\end{equation*}
A simple way for $q(\uv|\x_0)$ is uniform dequantization, where we set $q(\uv|\x_0)=\mathcal{U}(-\frac{1}{256},\frac{1}{256})$.

\section{Diffusion ODEs with Truncated-Normal Dequantization}
In this section, we discuss the challenges of training diffusion ODEs with dequantization and propose a training-free dequantization method for diffusion ODEs.

\subsection{Challenges for Diffusion ODEs with Dequantization}
\label{sec:challenges}
We first discuss the challenges for diffusion ODEs with dequantization in this section.
\paragraph{Truncation introduces an additional gap.}
Theoretically, we want to train diffusion ODEs by minimizing $\kl{q_0}{p_0}$ and use $p_0(\x_0)$ for the continuous model distribution. However, as $\sigma_0=0$, we have $\gamma_0=-\infty$. Due to this, it is shown in previous work~\citep{song2020score,kim2022soft} that there are numerical issues near $t=0$ for both training and sampling, so we cannot directly compute the model distribution $p_0$ at time $0$. In practice, a common solution is to choose a small starting time $\epsilon>0$ for improving numerical stability. The training objective then becomes minimizing $\kl{q_\epsilon}{p_\epsilon}$, which is equivalent to
\begin{equation}
\label{eqn:eps_mle}
\max_\theta \E_{q_0(\x_0)q_{0\epsilon}(\x_\epsilon|\x_0)}[\log p_\epsilon(\x_\epsilon)],
\end{equation}
and $\E_{q_0(\x_0)}\log p_\epsilon(\x_0)$ is directly used to evaluate the data likelihood. However, as $p_\epsilon \neq p_0$, such a method will introduce an additional gap due to the mismatch between training ($\E_{q_\epsilon(\x_\epsilon)}[\log p_\epsilon(\x_\epsilon)]$) and testing ($\E_{q_0(\x_0)}[\log p_\epsilon(\x_0)]$), which may degrade the likelihood evaluation performance.

\paragraph{Uniform dequantization causes a train-test mismatch.}
After choosing $\epsilon$, the continuous model distribution is defined by $p_{\text{model}}(\x) \coloneqq p_{\epsilon}(\x)$. Let $q(\uv|\x_0)$ be a dequantization distribution with support over $\uv\in [-\frac{1}{256},\frac{1}{256})^d$. The variational lower bound for the discrete model density $P_0(\x_0)$ is:
\begin{equation*}
\begin{split}
\E_{q_0(\x_0)}[\log P_0(\x_0)]
&\geq
\E_{q_0(\x_0)q(\uv|\x_0)}\left[
    \log p_\epsilon(\x_0 + \uv)
\right] \\
&- \E_{q_0(\x_0)q(\uv|\x_0)}\left[\log q(\uv|\x_0)
\right].
\end{split}
\end{equation*}
One widely used choice for $q(\uv|\x_0)$ is uniform distribution (uniform dequantization). However, this leads to a training-evaluation gap: for training, we fit $p_\epsilon$ to the distribution $q_\epsilon(\x_\epsilon)$, which is a Gaussian distribution near each discrete data point $\x_0$ because $\x_\epsilon = \alpha_\epsilon\x_0 + \sigma_\epsilon\epsilonv$ for $\epsilonv\sim \Nc(\vect{0},\Iv)$; while for evaluation, we test $p_\epsilon$ on uniform dequantized data $\x_0+\uv$. Such a gap will also degrade the likelihood evaluation performance and is not well-studied. 

In addition, another way for dequantization is to train a variational dequantization model $q_\phi(\uv|\x_0)$~\citep{ho2019flow++,song2021maximum} but it will need additional costs and is hard to train~\citep{kim2022soft}.

\subsection{Training-Free Dequantization by Truncated Normal}
In this section, we show that there exists a training-free dequantization distribution that fits diffusion ODEs well.

As discussed in Sec.~\ref{sec:challenges}, the gap between training and testing of diffusion ODEs is due to the difference between the training input $\x_\epsilon = \alpha_\epsilon \x_0 + \sigma_\epsilon \epsilonv$ (where $\epsilonv\sim\N(\vect{0},\Iv)$) and the testing input $\x_0 + \uv$. To fill such a gap, we can choose a dequantization distribution $q(\uv|\x_0)$ which satisfies
\begin{equation}
\label{eqn:condition_dequantization}
    \x_0 + \uv \approx \alpha_\epsilon\x_0 + \sigma_\epsilon\epsilonv,\quad \uv\in\left[-\frac{1}{256},\frac{1}{256}\right)^d.
\end{equation}
For small enough $\epsilon$, we have $\alpha_\epsilon\approx 1$, then Eqn.~\eqref{eqn:condition_dequantization} becomes $\uv \approx \frac{\sigma_\epsilon}{\alpha_\epsilon}\epsilonv$. We also need to ensure the support of $q(\uv|\x_0)$ is $[-\frac{1}{256},\frac{1}{256})^d$, i.e. the random variable $\frac{\sigma_\epsilon}{\alpha_\epsilon}\epsilonv$ is approximately within $[-\frac{1}{256},\frac{1}{256})^d$. To this end, we choose the variational dequantization distribution by a truncated normal distribution as follows:

\begin{equation}
\label{eqn:tn_dequant}
    q(\uv|\x_0)=\mathcal T\N(\mathbf 0,\frac{\sigma_\epsilon^2}{\alpha_\epsilon^2}\Iv,-\frac{1}{256},\frac{1}{256})
\end{equation}
where $\mathcal {TN}(\x|\muv,\sigma^2\Iv,a,b)$ is a truncated-normal distribution with mean $\muv$, covariance $\sigma^2\Iv$, and bounds $[a,b]$ in each dimension.
Moreover, such truncated-normal dequantization provides a guideline for choosing the start time $\epsilon$: To avoid large deviation from the truncation by $\frac{1}{256}$, we need to ensure that $\frac{\alpha_\epsilon}{\sigma_\epsilon}\uv \approx \epsilonv$ in most cases. We leverage the 3-$\sigma$ principle for standard normal distribution and let $\epsilon$ to satisfy $\frac{\alpha_\epsilon}{\sigma_\epsilon}\uv \in [-3, 3]^d$. As $\uv\in[-\frac{1}{256},\frac{1}{256})$, the critical start time $\epsilon$ satisfies that the negative log-SNR $\gamma_\epsilon = -\log\frac{\alpha_\epsilon^2}{\sigma_\epsilon^2}\approx -13.3$. Surprisingly, such choice of $\gamma_\epsilon$ is exactly the same as the $\gamma_{\text{min}}$ in \citet{kingma2021variational} which instead is obtained by training. Such dequantization distribution can ensure the conditions in Eqn.~\eqref{eqn:condition_dequantization} and we validate in Sec.~\ref{sec:experiments} that such dequantization can provide a tighter variational bound yet with no additional training costs. We summarize the likelihood evaluation by such dequantization distribution in the following theorem.

\begin{theorem}[Variational Bound under Truncated-Normal Dequantization]
Suppose we use the truncated-normal dequantization in Eqn.~\eqref{eqn:tn_dequant}, then the discrete model distribution has the following variational bound:
\begin{equation*}
\begin{aligned}
    \log P_0(\x_0)
&\geq
\E_{q(\hat\epsilonv)}\left[
    \log p_{\epsilon}(\hat\x_\epsilon)
\right] + \frac{d}{2}(1 + \log(2\pi\sigma_\epsilon^2))\\& + d\log Z - d\frac{\tau}{\sqrt{2\pi}Z}\exp(-\frac{1}{2}\tau^2)
\end{aligned}
\end{equation*}
where
\begin{equation*}
    \begin{aligned}
    \tau&=\frac{\alpha_\epsilon}{256\sigma_\epsilon}, \quad Z=\text{erf}\left(\frac{\tau}{\sqrt{2}}\right) \\
    \hat\x_\epsilon&= \alpha_\epsilon \x_0 + \sigma_\epsilon \hat\epsilonv,\quad \hat\epsilonv \sim \Tc\Nc\left(\hat\epsilonv\left| \vect{0},\Iv,-\tau,\tau\right.\right).
    \end{aligned}
\end{equation*}
\end{theorem}

Besides, we also have the following importance-weighted likelihood estimator by using $K$ i.i.d. samples by using Jensen's inequality as in \citet{burda2015importance}. As $K$ increases, the estimator gives a tighter bound, which enables 
more precise likelihood estimation.
\begin{corollary}[Importance Weighted Variational Bound under Truncated-Normal Dequantization]
Suppose we use the truncated-normal dequantization in Eqn.~\eqref{eqn:tn_dequant}, then the discrete model distribution has the following importance-weighted variational bound:
\begin{equation*}
\log\! P_0(\x_0)
\!\geq\!
\E_{\prod_{i=1}^K q(\hat\epsilonv^{(i)})}\!\!\left[
    \log \!\left(\!\!
        \frac{1}{K}\!\sum_{i=1}^K \frac{p_{\epsilon}(\hat\x_\epsilon^{(i)})}{q(\hat\epsilonv^{(i)})}\!\!
    \right)\!
\right] + d\log\sigma_\epsilon
\end{equation*}
where
\begin{equation*}
\begin{aligned}
\hat\x_\epsilon^{(i)} &= \alpha_\epsilon \x_0 + \sigma_\epsilon \hat\epsilonv^{(i)},\quad \hat\epsilonv^{(i)} \sim \Tc\Nc\left(\hat\epsilonv^{(i)}\left| \vect{0},\Iv,-\tau,\tau\right.\right)\\
q(\hat\epsilonv) &= \frac{1}{(2\pi Z^2)^{\frac{d}{2}}}\exp(-\frac{1}{2}\|\hat\epsilonv\|_2^2),\quad Z=\text{erf}\left(\frac{\tau}{\sqrt{2}}\right).
\end{aligned}
\end{equation*}
\end{corollary}

\begin{remark}
Another way to bridge the discrete-continuous gap is \textit{variational perspective}. We can view the process from discrete $\x_0$ to continuous $\x_\epsilon$ as a variational autoencoder, where the prior $p_{\epsilon}(\x_{\epsilon})$ is modeled by diffusion ODE. The dequantization and variational perspectives of diffusion ODEs have a close relationship both theoretically and empirically, and we detailedly discuss them in Appendix~\ref{appendix:perspective}.
\end{remark}

\section{Practical Techniques for Improving the Likelihood of Diffusion ODEs}
In this section, we propose some practical techniques for improving the likelihood of diffusion ODEs, including parameterization, a high-order training objective, and variance reduction by importance sampling. For simplicity, we denote $\dot{f_x} = \frac{\dm f(x)}{\dm x}$ for any scalar function $f(x)$.

\subsection{Velocity Parameterization}
\label{sec:v_param}

While the score matching objective $\Jc_{\SM}(\theta)$ only depends on the noise schedule, the training process is affected by many aspects such as network parameterization~\citep{song2020score,karras2022elucidating}. For example, the noise predictor $\epsilonv_\theta(\x_t,t)$ is widely used to replace the score predictor $\s_\theta(\x_t,t)$, since the noise $\epsilonv\sim\N(\mathbf 0,\Iv)$ has unit variance and is easier to fit, while $\s_\theta(\x_t,t)=-\epsilonv_\theta(\x_t,t)/\sigma_t$ is pathological and explosive near $t=0$~\citep{song2020score}.

In this work, we consider another network parameterization which is to directly predict the drift of the diffusion ODE. The parameterized model is defined by
\begin{equation}
\label{Eqn:v_ode}
    \frac{\dm\x_t}{\dm t} = \vv_\theta(\x_t,t) \coloneqq f(t)\x_t-\frac{1}{2}g^2(t)\s_\theta(\x_t,t)
\end{equation}
By rewriting the (first-order) score matching objective in Eqn.~\eqref{eqn:sm_raw}, $\Jc_{\SM}(\theta)$ is equivalent to:
\begin{equation}
\label{eqn:fm_1}
    \Jc_{\FM}(\theta)\coloneqq \int_0^T\frac{2}{g^2(t)}\E_{\x_0,\epsilonv}\left[\|\vv_\theta(\x_t,t)-\vv\|_2^2\right]\dt,
\end{equation}

where $\vv=\dot{\alpha }_{t}\x_{0} +\dot{\sigma }_{t}\epsilonv$ is the velocity to predict. Given unlimited model capacity, the optimal $\vv^*$ is
\begin{equation}
    \vv^*(\x_t,t)=f(t)\x_t-\frac{1}{2}g^2(t)\nabla_{\x}\log q_t(\x_t),
\end{equation}
which is the drift of probability flow ODE in Eqn.~\eqref{eqn:probability_flow_ODE}.

We give an intuitive explanation for $\Jc_{\text{FM}}$ in Appendix~\ref{appendix:v_pred_interpretation} that the prediction target $\vv$ is the tangent (velocity) of the diffusion path, and we name $\vv_\theta$ as \textit{velocity parameterization}. Besides, we show it empirically alleviates the \textit{imbalance problem} in noise prediction. 

In addition, we prove the equivalence between different predictors and different matching objectives for general noise schedules in Appendix~\ref{appendix:equivalence}. We also show in Appendix~\ref{appendix:relationship} that the flow matching method~\citet{lipman2022flow,albergo2022building,liu2022flow} and related techniques for improving the sample quality of diffusion models in~\citet{karras2022elucidating,salimans2022progressive,ho2022imagen} can all be reformulated in velocity parameterization. To be consistent, we still call $\Jc_{\text{FM}}$ as flow matching. It's an extended version of~\citet{lipman2022flow} with likelihood weighting and several practical modifications as detailed in Section~\ref{sec:practical_consideration}.

\subsection{Error-bounded Second-Order Flow Matching}
\label{sec:high_order}

According to \citet{chen2018neural}, the ODE likelihood of Eqn.~\eqref{Eqn:v_ode} can be evaluated by solving the following differential equation from $\epsilon$ to $T$:
\begin{equation}
    \frac{\dm\log p_t(\x_t)}{\dt}=-\tr(\nabla_{\x}\vv_\theta(\x_t,t)).
\end{equation}
As $\Jc_{\FM}$ in Eqn.~\eqref{eqn:fm_1} can only restrict the distance between $\vv_\theta$ and $\vv^*$, but not the divergence $\tr(\nabla_{\x}\vv_\theta)$ and $\tr(\nabla_{\x}\vv^*)$.
The precision and smoothness of the trace $\tr(\nabla_{\x}\vv_\theta(\x_t,t))$ affects the likelihood performance and the number of function evaluations for sampling. For simulation-free training of $\tr(\nabla_{\x}\vv_\theta(\x_t,t))$, we propose an error-bounded trace of second-order flow matching, where the second-order error is bounded by the proposed objective and first-order error.

\begin{theorem}
\label{thrm:second_trace}
(Error-Bounded Trace of Second-Order Flow Matching) Suppose we have a first-order velocity estimator $\hat\vv_1(\x_t,t)$, we can learn a second-order trace velocity model $\vv_2^{\text{trace}}(\cdot,t;\theta):\R^{d}\rightarrow\R$ which minimizes
\begin{equation*}
    \E_{q_t(\x_t)}\left[\left|\vv_2^{\text{trace}}(\x_t,t;\theta)-\tr(\nabla_{\x}\vv^*(\x_t,t))\right|^2\right],
\end{equation*}
by optimizing
{\small
\begin{equation}
\label{eqn:dsm-2-trace-obj}
    \theta^*\!=\argmin_{\theta}\E_{\x_0,\epsilonv}\!\left[\left|\vv_2^{\text{trace}}(\x_t,t;\theta)\!-\!\frac{\dot{\sigma}_t}{\sigma_t}d\!+\!\ellv_1\right|^2\right].
\end{equation}
}%
where
\begin{equation*}
\begin{aligned}
    &\ellv_1(\epsilonv,\x_0,t)\coloneqq\frac{2}{g^2(t)}\|\hat\vv_1(\x_t,t)-\vv\|_2^2,\\
    &\x_t=\alpha_t\x_0+\sigma_t\epsilonv,\quad \vv=\dot{\alpha}_t\x_0+\dot{\sigma}_t\epsilonv,\quad\epsilonv\sim \Nc(\vect{0},\Iv).
\end{aligned}
\end{equation*}
Moreover, denote the first-order flow matching error as $\delta_1(\x_t,t)\coloneqq\|\hat\vv_1(\x_t,t)-\vv^*(\x_t,t)\|_2$, then $\forall \x_t,  \theta$, the estimation error for $\vv_2^{\text{trace}}(\x_t,t;\theta)$ can be bounded by:
\begin{equation*}
 \begin{split}
&\left|\vv_2^{\text{trace}}(\x_t,t;\theta)-\tr(\nabla_{\x}\vv^*(\x_t,t))\right|\\
\leq\ &\left|\vv_2^{\text{trace}}(\x_t,t;\theta)-\vv_2^{\text{trace}}(\x_t,t;\theta^*)\right|+\frac{2}{g^2(t)}\delta_1^2(\x_t,t).
 \end{split} 
\end{equation*}
\end{theorem}

The proof is provided in Appendix~\ref{appendix:trace}. In practice, we choose $\vv_2^{\text{trace}}(\x_t,t;\theta)=\tr(\nabla_{\x}\vv_\theta(\x_t,t))$ for self-regularizing.
As for scalability, we use Hutchinson's trace estimator~\citep{hutchinson1990stochastic} to unbiasedly estimate the trace, and use forward-mode automatic differentiation to compute Jacobian-vector product~\citep{lu2022maximum}.

\subsection{Timing by Log-SNR and Normalizing Velocity}
\label{sec:practical_consideration}
In practice, we make two modifications to improve the performance. First, we use negative log-SNR $\gamma_t$ to time the diffusion process. Still, we parameterize $\vv_\theta(\x_\gamma,\gamma)$ to predict the drift of the $\gamma$ timed diffusion ODE i.e. $\frac{\dm \x_\gamma}{\dm\gamma}=\vv_\theta(\x_\gamma,\gamma)$, so the corresponding predictor $\vv_\theta(\x_t,t)=\vv_\theta(\x_\gamma,\gamma)\frac{\dm\gamma}{\dt}$. Second, the velocity of the diffusion path $\vv=\dot{\alpha }_{t}\x_{0} +\dot{\sigma }_{t}\epsilonv$ may have different scales at different $t$, so we propose to predict the normalized velocity $\tilde\vv=\vv/\sqrt{\dot{\alpha }_{t}^2+\dot{\sigma }_{t}^2}$, with the parameterized network $\tilde{\vv}_\theta(\x_t,t)=\vv_\theta(\x_t,t)/\sqrt{\dot{\alpha }_{t}^2+\dot{\sigma }_{t}^2}$, which is equal to $\tilde{\vv}_\theta(\x_\gamma,\gamma)=\vv_\theta(\x_\gamma,\gamma)/\sqrt{\dot{\alpha }_{\gamma}^2+\dot{\sigma }_{\gamma}^2}$. The objective in Eqn.~\eqref{eqn:fm_1} reduces to
\begin{equation*}
    \Jc_{\FM}(\theta)\!=\!\int_{\gamma_0}^{\gamma_T}\!\! 2\frac{\dot{\alpha}_\gamma^2+\dot{\sigma}_\gamma^2}{\sigma_\gamma^2}\E_{\x_0,\epsilonv}\|\tilde{\vv}_\theta(\x_\gamma,\gamma)-\tilde{\vv}\|_2^2\dm \gamma.
\end{equation*}

And the corresponding second-order objective:
\begin{align}
\label{eqn:second_order_objective}
    \Jc_{\FM,\tr}\!=&\!\int_{\gamma_0}^{\gamma_T}\!\! 2\frac{\dot{\alpha}_\gamma^2+\dot{\sigma}_\gamma^2}{\sigma_\gamma^2}\E_{\x_0,\epsilonv}\!\Bigg(\!\sigma_\gamma\tr(\nabla\tilde\vv_\theta) \!-\! \frac{\dot{\sigma}_\gamma}{\sqrt{\dot{\alpha}_\gamma^2+\dot{\sigma}_\gamma^2}}d  \nonumber\\
    & +\frac{2\sqrt{\dot{\alpha}_\gamma^2+\dot{\sigma}_\gamma^2}}{\sigma_\gamma}\|\tilde\vv^{(s)}_\theta(\x_\gamma,\gamma)-\tilde\vv\|_2^2\Bigg)^2\dm\gamma
\end{align}

where $\tilde{\vv}^{(s)}_\theta$ is the stop-gradient version of $\tilde\vv_\theta$, since we only use the parameterized first-order velocity predictor as an estimator. Our final formulation of parameterized diffusion ODE is
\begin{equation}
    \frac{\dm \x_\gamma}{\dm\gamma}=\sqrt{\dot{\alpha}_\gamma^2+\dot{\sigma}_\gamma^2}\tilde\vv_\theta(\x_\gamma,\gamma)
\end{equation}

\subsection{Variance Reduction with Importance Sampling}
\label{sec:IS}

The flow matching is conducted for all $\gamma$ in $[\gamma_0,\gamma_T]$ through an integral. In practice, the evaluation of the integral is time-consuming, and Monte-Carlo methods are used to unbiasedly estimate the objective by uniformly sampling $\gamma$. In this case, the variance of the Monte-Carlo estimator affects the optimization process. Thus, a continuous importance distribution $p(\gamma)$ can be proposed for variance reduction. Denote $\mathcal{L}_\theta(\x_0,\epsilonv,\gamma,)=2\frac{\dot{\alpha}_\gamma^2+\dot{\sigma}_\gamma^2}{\sigma_\gamma^2}\|\tilde{\vv}_\theta(\x_\gamma,\gamma)-\tilde{\vv}\|_2^2$, then

\begin{equation}
\Jc_{\FM}(\theta)
=\E_{\gamma\sim p(\gamma)}\E_{\x_0,\epsilonv}\left[\frac{\mathcal{L}_\theta(\x_0,\epsilonv,\gamma)}{p(\gamma)}\right]
\end{equation}
We propose to use two types of importance sampling (IS), and empirically compare them for faster convergence.

\paragraph{Designed IS}
Intuitively, we can choose $p(\gamma)\propto \frac{\dot{\alpha}_\gamma^2+\dot{\sigma}_\gamma^2}{\sigma_\gamma^2}$. This way, the coefficients of $\|\tilde{\vv}_\theta(\x_\gamma,\gamma)-\tilde{\vv}\|_2^2$ is a time-invariant constant, and the velocity matching error is not amplified or shrank at any $\gamma$. This is similar to the IS in \citet{song2021maximum}, where the $g^2(t)/\sigma_t^2$ weighting before the noise matching error $\|\epsilonv_\theta(\x_t,t)-\epsilonv\|_2^2$ is cancelled, and it corresponds to uniform $\gamma$ under our parameterization.

For noise schedules used in this paper, we can obtain closed-form sampling procedures using \textit{inverse transform sampling}, see Appendix~\ref{appendix:specification}.
\paragraph{Learned IS} The variance of the Monte-Carlo estimator depends on the learned network $\tilde{\vv}_\theta$. To minimize the variance, we can parameterize the IS with another network and treat the variance as an objective. Actually, learning $p(\gamma)$ is equivalent to learning a monotone mapping $\gamma(t):[0,1]\rightarrow [\gamma_0,\gamma_T]$, which is inverse cumulative distribution function of $p(\gamma)$. We can uniformly sample $t$, and regard the IS as change-of-variable from $\gamma$ to $t$.
\begin{equation}
    \Jc_{\FM}(\theta)
=\E_{t\sim \Uc(0,1)}\E_{\x_0,\epsilonv}\left[\gamma'(t)\mathcal{L}_\theta(\x_0,\epsilonv,\gamma(t))\right]
\end{equation}

Suppose we parameterize $\gamma(t)$ with $\eta$. Denote $\mathcal{L}_{\theta,\eta}(\x_0,\epsilonv,t)=\gamma_\eta'(t)\mathcal{L}_\theta(\x_0,\epsilonv,\gamma_\eta(t))$, which is a Monte-Carlo estimator of $\Jc_{\FM}(\theta)$. Since its variance $\Var_{t,\epsilonv,\x_0}[\mathcal{L}_{\theta,\eta}(\x_0,\epsilonv,t)]=\E_{t,\epsilonv,\x_0}[\mathcal{L}_{\theta,\eta}^2(\x_0,\epsilonv,t)]-\Jc_{\FM}^2(\theta)$ and $\Jc_{\FM}(\theta)$ is invariant to $\gamma_\eta(t)$, we can minimize $\E_{t,\epsilonv,\x_0}[\mathcal{L}_{\theta,\eta}^2(\x_0,\epsilonv,t)]$ for variance reduction.

While this approach seeks the optimal IS, it causes extra overhead by introducing an IS network, requiring complex gradient operation or additional training steps. Thus, we only use it as a reference to test the optimality of our designed IS. We simplify the variance reduction in \citet{kingma2021variational}, and propose an \textit{adaptive IS} algorithm, which is detailed in Appendix~\ref{appendix:is}. Empirically, we show that designed IS is a more preferred approach since it is training-free and achieves a similar convergence speed to learned IS.

\section{Related Work}
Diffusion models, also known as score-based generative models (SGMs), have achieved state-of-the-art sample quality and likelihood~\citep{dhariwal2021diffusion,karras2022elucidating,kingma2021variational} among deep generative models, yielding extensive downstream applications such as speech and singing synthesis~\citep{chen2020wavegrad,liu2022diffsinger}, conditional image generation~\citep{ramesh2022hierarchical,rombach2022high}, guided image editing~\citep{meng2021sdedit,nichol2021glide}, unpaired image-to-image translation~\citep{zhao2022egsde} and inverse problem solving~\citep{chung2022diffusion,kawar2022denoising}.

Diffusion ODEs are special formulations of neural ODEs and can be viewed as continuous normalizing flows~\citep{chen2018neural}. Training of diffusion ODEs can be categorized into simulation-based and simulation-free methods. The former utilizes the exact likelihood evaluation formula of ODE~\citep{chen2018neural}, which leads to a maximum likelihood training procedure~\citep{grathwohl2018ffjord}. However, it involves expensive ODE simulations for forward and backward propagation and may result in unnecessary complex dynamics~\citep{finlay2020train} since it only cares about the model distribution at $t=0$. The latter trains neural ODEs by matching their trajectories to a predefined path, such as the diffusion process. This approach is proposed in \citet{song2020score}, and extended in \citet{lu2022maximum,lipman2022flow,albergo2022building,liu2022flow}. We propose velocity parameterization which is an extension of~\citet{lipman2022flow} with practical modifications and claim that the paths used in \citet{lipman2022flow,albergo2022building,liu2022flow} are special cases of noise schedule. Aiming at maximum likelihood training, we also get inspiration from \citet{lu2022maximum}. We additionally apply likelihood weighting and propose to finetune the model with high-order flow matching.

Variance reduction techniques are commonly used for training diffusion models. \citet{nichol2021improved} proposes an importance sampling (IS) for discrete-time diffusion models by maintaining the historical losses at each time step and building the proposal distribution based on them. \citet{song2021maximum} designs an IS to cancel out the weighting before the noise matching loss. \citet{kingma2021variational} proposes a variance reduction method that is equivalent to learning a parameterized IS. We simply their procedure and propose an adaptive IS scheme for ablation. By empirically comparing different IS methods, we find a designed and analytical IS distribution that achieves a good performance-efficiency trade-off.

\section{Experiments}
\label{sec:experiments}
In this section, we present our training procedure and experiment settings, and our ablation studies to demonstrate how our techniques improve the likelihood of diffusion ODEs.

We implement our methods based on the open-source codebase of \citet{kingma2021variational} implemented with JAX~\citet{bradbury2018jax}, and use similar network and hyperparameter settings. We first train the model by optimizing our first-order flow matching objective $\min_\theta\Jc_{\FM}(\theta)$ for enough iterations, so that the first-order velocity prediction has little error. Then, we finetune the pretrained first-order model using a mixture of first-order and second-order flow matching objectives $\min_\theta\Jc_{\FM}(\theta)+\lambda\Jc_{\FM,\tr}(\theta)$. The finetune process converges in much fewer iterations than pretraining. Finally, we evaluate the likelihood on the test set using the variational bound under our proposed truncated-normal dequantization. The detailed training configurations are provided in Appendix~\ref{appendix:details}.

Our training and evaluation procedure is feasible for any noise schedule $\alpha_\gamma,\sigma_\gamma$. We choose two special noise schedules:
\paragraph{Variance Preserving (VP)} $\alpha_\gamma^2+\sigma_\gamma^2=1$. This schedule is widely used in diffusion models, which yields a process with a fixed variance of one when the initial
distribution has a unit variance.
\paragraph{Straight Path (SP)} $\alpha_\gamma+\sigma_\gamma=1$. This schedule is used in \citet{lipman2022flow,albergo2022building,liu2022flow}, where they call it OT path and claim it leads to better dynamics since the pairwise diffusion paths are straight lines. We simply regard it as a special kind of noise schedule.

Under these two schedules, $\alpha_\gamma,\sigma_\gamma$ are uniquely determined by $\gamma$, and we do not have any extra hyperparameters. They also have corresponding objectives and designed IS, which can be expressed in closed form (see Appendix~\ref{appendix:specification} for details). We train our i-DODE on CIFAR-10~\citep{krizhevsky2009learning} and ImageNet-32\footnote{There are two different versions of ImageNet32 and ImageNet64 datasets. For fair comparisons, we use both versions of ImageNet32, one is downloaded from \url{https://image-net.org/data/downsample/Imagenet32_train.zip}, following~\citet{lipman2022flow}, and the other is downloaded from \url{http://image-net.org/small/train_32x32.tar} (old version, no longer available), following \citet{song2021maximum} and \citet{kingma2021variational}. The former dataset applies anti-aliasing and is easier for maximum likelihood training.}~\citep{deng2009imagenet}, which are two popular benchmarks for generative modeling and density estimation.

\subsection{Likelihood and Samples}

\begin{table*}[t]
    \vspace{-.1in}
    \centering
    \caption{\label{tab:result}\small{Negative log-likelihood (NLL) in bits/dim (BPD), sample quality (FID scores) and number of function evaluations (NFE) on CIFAR-10 and ImageNet 32x32. For fair comparisons, we list NLL results of previous ODEs without variational dequantization or data augmentation (unless specifically stated), and FID/NFE results obtained by adaptive-step ODE solver. Results with ``$\slash$'' mean they are not reported in the original papers or do not apply. $^\dagger$For VDM, since they have no ODE formulation, the FID score is obtained by 1000 step discretization of their SDE. We report their corresponding ODE result in the ablation study. $^*$Corresponding to the old version ImageNet-32 dataset.}}
    \vskip 0.1in
    \begin{small}
    \begin{tabular}{lcccccc}
    \toprule
    Model & \multicolumn{3}{c}{CIFAR-10} & \multicolumn{3}{c}{ImageNet-32} \Bstrut \\
    \cline{2-7}
    & NLL $\downarrow$ & FID $\downarrow$&NFE $\downarrow$ & NLL $\downarrow$ & FID $\downarrow$&NFE $\downarrow$ \Tstrut \\
    \midrule
        VDM~\citep{kingma2021variational} & 2.65 &7.60$^\dagger$&1000 & 3.72$^*$&$\slash$&$\slash$ \Bstrut \\
        VDM (with data augmentation) ~\citep{kingma2021variational} & 2.49 &$\slash$&$\slash$ & $\slash$&$\slash$&$\slash$ \Bstrut \\
    \hline
    \textit{(Previous ODE)}&&&&&& \Tstrut \\
    FFJORD~\citep{grathwohl2018ffjord}&3.40& $\slash$ & $\slash$ & $\slash$ & $\slash$ & $\slash$\\
    ScoreSDE~\citep{song2020score}&2.99&2.92& $\slash$ & $\slash$ & $\slash$ & $\slash$\\
    ScoreFlow~\citep{song2021maximum}&2.90&5.40& $\slash$ &3.82$^*$&10.18$^*$&$\slash$\\
    Soft Truncation~\citep{kim2022soft}&3.01&3.96&$\slash$&3.90$^*$&8.42$^*$&$\slash$\\
    Flow Matching~\citep{lipman2022flow}&2.99&6.35&142&3.53&5.31&122\\
    Stochastic Interp.\citep{albergo2022building}&2.99&10.27&$\slash$&3.48&8.49&$\slash$ \Bstrut\\
    
    \hline
        i-DODE (SP)~(\bf{ours}) & 2.56 &11.20 &162 &3.44/\textbf{3.69}$^*$&10.31&138 \Tstrut\\
        i-DODE (VP)~(\bf{ours}) & 2.57 &10.74&126  & \textbf{3.43}/3.70$^*$&9.09&152\\
        i-DODE (VP, with data augmentation)~(\bf{ours}) & \bf{2.42} &3.76& 215 & $\slash$&$\slash$&$\slash$\\
    \bottomrule
    \end{tabular}
    \end{small}
    \vspace{-0.15in}
\end{table*}

\begin{figure}[ht]
	\centering		
 	\includegraphics[width=.9\linewidth]{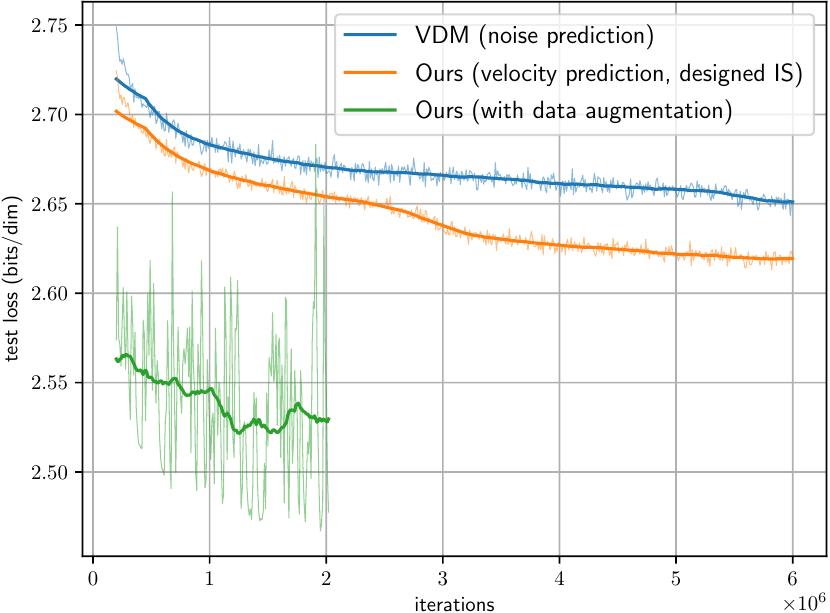}\\
 	\vspace{-.1in}
	\caption{\label{fig:test_curve} Test loss curve in the pretraining phase, compared to VDM~\cite{kingma2021variational}. We compute the loss on the test set by the SDE likelihood bound in \citet{kingma2021variational}.}
 	\vspace{-.1in}
\end{figure}

Table~\ref{tab:result} shows our experiment results on CIFAR-10 and ImageNet-32 datasets. Our models are pretrained with velocity parameterization, designed IS, and finetuned with second-order flow matching. We report the likelihood values using our truncated-normal dequantization with the importance-weighted estimator under $K=20$. To compute the FID values, we apply an adaptive-step ODE solver to draw samples from the diffusion ODEs. We also report the NFE during the sampling process, which reflects the smoothness of the dynamics. 

Combining our training techniques and dequantization, we exceed the likelihood of previous ODEs, especially by a large margin on CIFAR-10. In Figure~\ref{fig:test_curve}, we compare our pretraining phase to VDM~\citep{kingma2021variational}, which indicates that our techniques achieve 2x$\sim$3x times of previous convergence speed. We further strengthen the likelihood results by employing data augmentation techniques and a larger network, following VDM. We observe that augmented training data may cause fluctuations in the training and testing losses. When we select the models that achieve the best testing performance, we obtain an SDE likelihood of 2.46 at around only 2M iterations, compared to 2.49 of VDM at 10M iterations.

We do not observe the superiority
of SP to VP such as lower FID and NFE as in~\citet{lipman2022flow}. We suspect it may result from maximum likelihood training, which puts more emphasis on the high log-SNR region. More theoretical comparisons with~\citet{lipman2022flow} are given in Appendix~\ref{appendix:comparison_flow_matching}.

Randomly generated samples from our models are provided in Appendix~\ref{appendix:samples}. Since we use network architecture and techniques targeted at the likelihood, our FID is worse than the state-of-the-art, which can be improved by designing time weighting to emphasize the training at small log-SNR levels~\citep{kingma2021variational} or using high-quality sampling algorithms such as PC sampler~\citep{song2020score}. Besides, data augmentation and a larger network notably improve the FID to 3.76 on CIFAR-10, while achieving the state-of-the-art likelihood.

\subsection{Ablations}

\begin{figure}[t]
	\centering		
 	\includegraphics[width=.8\linewidth]{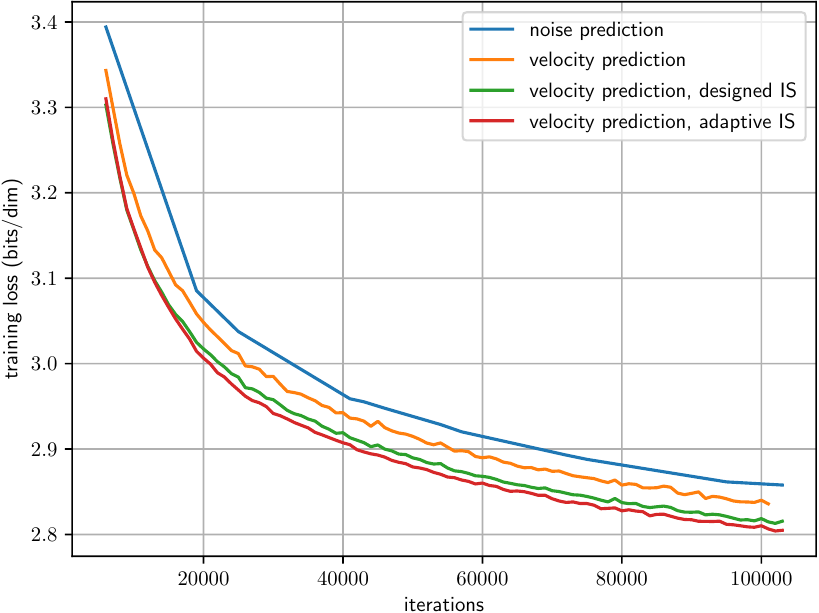}\\
 	\vspace{-.1in}
	\caption{\label{fig:training_curve_ablation}Training curve from scratch for ablation. We compute the loss on the training set by the SDE likelihood bound in \citet{kingma2021variational}.}
 	\vspace{-.1in}
\end{figure}

\begin{figure}[t]

	\centering
	\begin{minipage}{.48\linewidth}
		\centering
			\includegraphics[width=.9\linewidth]{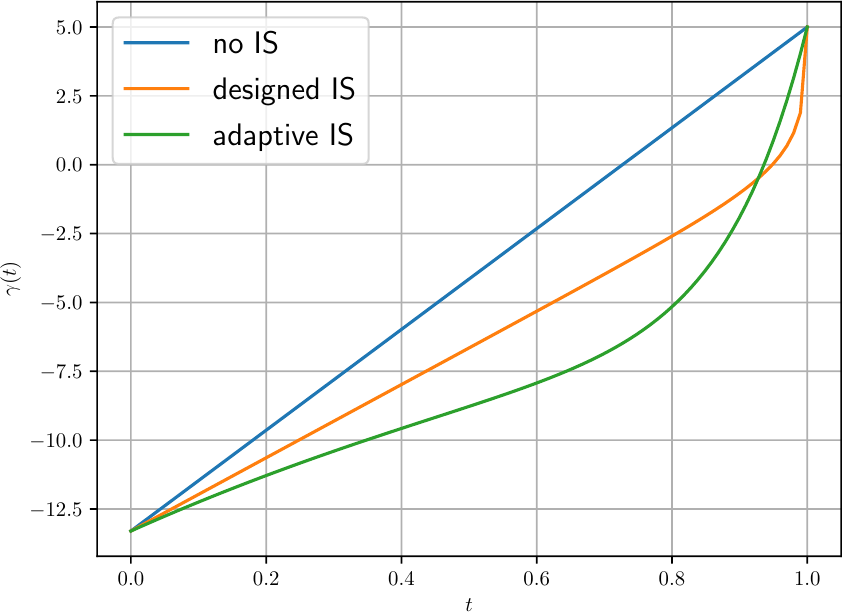}\\
\small{(a) $\gamma(t)$}
	\end{minipage}
	\begin{minipage}{.48\linewidth}
	\centering
	\includegraphics[width=.9\linewidth]{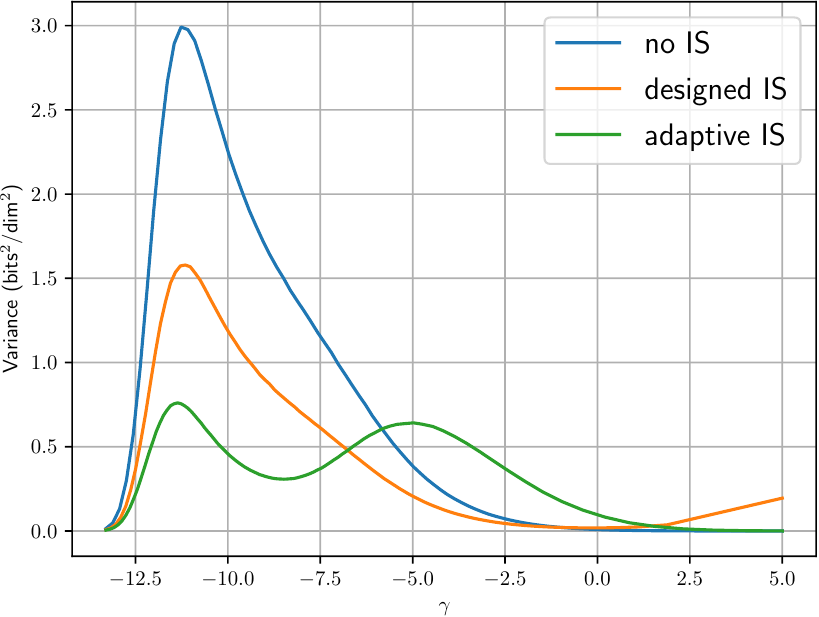}\\
\small{	(b) Variance at different log-SNR levels.}
\end{minipage}
   \vspace{-.05in}
	\caption{\label{fig:vis_is}Visualization of importance sampling: (a) The inverse cumulative distribution function $\gamma(t)$ of the proposal distribution $p(\gamma)$, which maps uniform $t$ to importance sampled $\gamma$ (b) The variance of Monte-Carlo estimator $\Var\left[\gamma'(t)\mathcal{L}_\theta(\x_0,\epsilonv,\gamma(t))\right]$ at different noise levels, estimated using 32 data samples $\x_0$ and 100 noise samples $\epsilonv$. The peak variance is achieved around $\gamma=-11.2$.}
	\vspace{-.15in}
\end{figure}

Due to the expensive time cost of pretraining, we only conduct ablation studies on CIFAR-10 under the VP schedule. First, we test our techniques for pretraining when training from scratch. We plot the training curves with noise predictor~\citep{kingma2021variational} and velocity predictor, then further implement our IS strategies (Figure~\ref{fig:training_curve_ablation}). We find that velocity parameterization and IS both accelerate the training process, while designed IS performs slightly worse than adaptive IS. Considering the extra time cost for learning the IS network, we conclude that designed IS is a better choice for large-scale pretraining. Then we visualize different IS by plotting the mapping from uniform $t$ to importance sampled $\gamma$, as well as the variance at different noise levels on the pretrained model (Figure~\ref{fig:vis_is}). We show that the IS reduces the variance by sampling more in high log-SNR regions.

\begin{table}[ht]

    \vspace{-.1in}
    \centering
    \caption{\label{tab:ablation_converged}Ablation study when converged. We report negative log-likelihood (NLL) in bits/dim (BPD), sample quality (FID scores), and number of function evaluations (NFE) after our pretraining and finetuning phase. We evaluate NLL by uniform (U) and truncated-normal (TN) dequantization without importance weight. We retrain VDM and evaluate its ODE form.}
    \vskip 0.1in
    \begin{small}
    \resizebox{0.48\textwidth}{!}{
    \begin{tabular}{lcccc}
    \toprule
    Model &  NLL (U) & NLL (TN) & FID & NFE\\
    \midrule
       VDM~\citep{kingma2021variational}  & 2.78 & 2.64 &8.65 &213\\
       Pretrain~(\bf{ours}) & 2.75 & 2.61 &10.66 &248\\
       + Finetune~(\bf{ours}) & 2.74 & \bf{2.60} & 10.74 &\bf{126}\\
    \bottomrule
    \end{tabular}
    }
    \end{small}
\end{table}

Next, we test our pretraining, finetuning and evaluation on the converged model (Table~\ref{tab:ablation_converged}). As stated before, our pretraining has faster loss descent and converges to a higher likelihood than VDM. Based on it, our finetuning slightly improves the ODE likelihood and smooths the flow, leading to much less NFE when sampling. Our truncated-normal dequantization is also a key factor for precise likelihood computing, which surpasses previous uniform dequantization by a large margin.

In agreement with \citet{song2021maximum}, our improvements in likelihood lead to slightly worse FIDs. We also argue that the degeneration is small in terms of visual quality. We provide additional samples in Appendix~\ref{appendix:samples} for comparison. 

\section{Conclusion}
We propose improved techniques for simulation-free maximum likelihood training and likelihood evaluation of diffusion ODEs. Our training stage involves improved pretraining and additional finetuning, which results in fast convergence, high likelihood and smooth trajectory. We improve the likelihood evaluation with novel truncated-normal dequantization, which is training-free and tailored for diffusion ODEs. Empirically, we achieve state-of-the-art likelihood on image datasets without variational dequantization or data augmentation and make a breakthrough on CIFAR-10 compared to previous ODEs. Due to resource limitations, we didn't explore tuning of hyperparameters and network architectures, which are left for future work.

\section*{Acknowledgements}
This work was supported by the National Key Research and Development Program of China
(2020AAA0106302); 
NSF of China Projects (Nos. 62061136001, 61620106010, 62076145, U19B2034, U1811461, U19A2081, 6197222, 62106120, 62076145); a grant from Tsinghua Institute for Guo Qiang; the High Performance Computing Center, Tsinghua University. J.Z was also supported by the New Cornerstone Science Foundation through the XPLORER PRIZE. The large-scale training was supported by Shengshu Technology.

\bibliography{example_paper}

\begin{thebibliography}{61}
\providecommand{\natexlab}[1]{#1}
\providecommand{\url}[1]{\texttt{#1}}
\expandafter\ifx\csname urlstyle\endcsname\relax
  \providecommand{\doi}[1]{doi: #1}\else
  \providecommand{\doi}{doi: \begingroup \urlstyle{rm}\Url}\fi

\bibitem[Albergo \& Vanden-Eijnden(2022)Albergo and
  Vanden-Eijnden]{albergo2022building}
Albergo, M.~S. and Vanden-Eijnden, E.
\newblock Building normalizing flows with stochastic interpolants.
\newblock In \emph{The Eleventh International Conference on Learning
  Representations}, 2022.

\bibitem[Anderson(1982)]{anderson1982reverse}
Anderson, B.~D.
\newblock Reverse-time diffusion equation models.
\newblock \emph{Stochastic Processes and their Applications}, 12\penalty0
  (3):\penalty0 313--326, 1982.

\bibitem[Bradbury et~al.(2018)Bradbury, Frostig, Hawkins, Johnson, Leary,
  Maclaurin, Necula, Paszke, VanderPlas, Wanderman-Milne,
  et~al.]{bradbury2018jax}
Bradbury, J., Frostig, R., Hawkins, P., Johnson, M.~J., Leary, C., Maclaurin,
  D., Necula, G., Paszke, A., VanderPlas, J., Wanderman-Milne, S., et~al.
\newblock Jax: composable transformations of python+ numpy programs.
\newblock \emph{Version 0.2}, 5:\penalty0 14--24, 2018.

\bibitem[Burda et~al.(2015)Burda, Grosse, and
  Salakhutdinov]{burda2015importance}
Burda, Y., Grosse, R., and Salakhutdinov, R.
\newblock Importance weighted autoencoders.
\newblock \emph{arXiv preprint arXiv:1509.00519}, 2015.

\bibitem[Chen et~al.(2021)Chen, Zhang, Zen, Weiss, Norouzi, and
  Chan]{chen2020wavegrad}
Chen, N., Zhang, Y., Zen, H., Weiss, R.~J., Norouzi, M., and Chan, W.
\newblock Wavegrad: Estimating gradients for waveform generation.
\newblock In \emph{International Conference on Learning Representations}, 2021.

\bibitem[Chen et~al.(2018{\natexlab{a}})Chen, Rubanova, Bettencourt, and
  Duvenaud]{chen2018neural}
Chen, R.~T., Rubanova, Y., Bettencourt, J., and Duvenaud, D.
\newblock Neural ordinary differential equations.
\newblock In \emph{Proceedings of the 32nd International Conference on Neural
  Information Processing Systems}, pp.\  6572--6583, 2018{\natexlab{a}}.

\bibitem[Chen et~al.(2018{\natexlab{b}})Chen, Mishra, Rohaninejad, and
  Abbeel]{chen2018pixelsnail}
Chen, X., Mishra, N., Rohaninejad, M., and Abbeel, P.
\newblock Pixelsnail: An improved autoregressive generative model.
\newblock In \emph{International Conference on Machine Learning}, pp.\
  864--872. PMLR, 2018{\natexlab{b}}.

\bibitem[Chen et~al.(2018{\natexlab{c}})Chen, Yeo, Lee, and
  Lau]{chen2018autoencoder}
Chen, Z., Yeo, C.~K., Lee, B.~S., and Lau, C.~T.
\newblock Autoencoder-based network anomaly detection.
\newblock In \emph{2018 Wireless telecommunications symposium (WTS)}, pp.\
  1--5. IEEE, 2018{\natexlab{c}}.

\bibitem[Choi et~al.(2022)Choi, Meng, Song, and Ermon]{choi2022density}
Choi, K., Meng, C., Song, Y., and Ermon, S.
\newblock Density ratio estimation via infinitesimal classification.
\newblock In \emph{International Conference on Artificial Intelligence and
  Statistics}, pp.\  2552--2573. PMLR, 2022.

\bibitem[Chung et~al.(2022)Chung, Kim, Mccann, Klasky, and
  Ye]{chung2022diffusion}
Chung, H., Kim, J., Mccann, M.~T., Klasky, M.~L., and Ye, J.~C.
\newblock Diffusion posterior sampling for general noisy inverse problems.
\newblock In \emph{The Eleventh International Conference on Learning
  Representations}, 2022.

\bibitem[Deng et~al.(2009)Deng, Dong, Socher, Li, Li, and
  Fei{-}Fei]{deng2009imagenet}
Deng, J., Dong, W., Socher, R., Li, L., Li, K., and Fei{-}Fei, L.
\newblock Image{N}et: {A} large-scale hierarchical image database.
\newblock In \emph{2009 IEEE Conference on Computer Vision and Pattern
  Recognition}, pp.\  248--255. IEEE, 2009.

\bibitem[Dhariwal \& Nichol(2021)Dhariwal and Nichol]{dhariwal2021diffusion}
Dhariwal, P. and Nichol, A.~Q.
\newblock Diffusion models beat {GAN}s on image synthesis.
\newblock In \emph{Advances in Neural Information Processing Systems},
  volume~34, pp.\  8780--8794, 2021.

\bibitem[Dias et~al.(2020)Dias, Mattos, da~Silva, de~Macedo, and
  Silva]{dias2020anomaly}
Dias, M.~L., Mattos, C. L.~C., da~Silva, T.~L., de~Macedo, J. A.~F., and Silva,
  W.~C.
\newblock Anomaly detection in trajectory data with normalizing flows.
\newblock In \emph{2020 International Joint Conference on Neural Networks
  (IJCNN)}, pp.\  1--8. IEEE, 2020.

\bibitem[Dinh et~al.(2017)Dinh, Sohl-Dickstein, and Bengio]{dinh2016density}
Dinh, L., Sohl-Dickstein, J., and Bengio, S.
\newblock Density estimation using real nvp.
\newblock In \emph{International Conference on Learning Representations}, 2017.

\bibitem[Dormand \& Prince(1980)Dormand and Prince]{dormand1980family}
Dormand, J.~R. and Prince, P.~J.
\newblock A family of embedded {R}unge-{K}utta formulae.
\newblock \emph{Journal of computational and applied mathematics}, 6\penalty0
  (1):\penalty0 19--26, 1980.

\bibitem[Finlay et~al.(2020)Finlay, Jacobsen, Nurbekyan, and
  Oberman]{finlay2020train}
Finlay, C., Jacobsen, J.-H., Nurbekyan, L., and Oberman, A.
\newblock How to train your neural ode: the world of jacobian and kinetic
  regularization.
\newblock In \emph{International conference on machine learning}, pp.\
  3154--3164. PMLR, 2020.

\bibitem[Grathwohl et~al.(2019)Grathwohl, Chen, Bettencourt, Sutskever, and
  Duvenaud]{grathwohl2018ffjord}
Grathwohl, W., Chen, R.~T., Bettencourt, J., Sutskever, I., and Duvenaud, D.
\newblock Ffjord: Free-form continuous dynamics for scalable reversible
  generative models.
\newblock In \emph{International Conference on Learning Representations}, 2019.

\bibitem[Helminger et~al.(2020)Helminger, Djelouah, Gross, and
  Schroers]{helminger2020lossy}
Helminger, L., Djelouah, A., Gross, M., and Schroers, C.
\newblock Lossy image compression with normalizing flows.
\newblock \emph{arXiv preprint arXiv:2008.10486}, 2020.

\bibitem[Ho et~al.(2019)Ho, Chen, Srinivas, Duan, and Abbeel]{ho2019flow++}
Ho, J., Chen, X., Srinivas, A., Duan, Y., and Abbeel, P.
\newblock Flow++: Improving flow-based generative models with variational
  dequantization and architecture design.
\newblock In \emph{International Conference on Machine Learning}, pp.\
  2722--2730. PMLR, 2019.

\bibitem[Ho et~al.(2020)Ho, Jain, and Abbeel]{ho2020denoising}
Ho, J., Jain, A., and Abbeel, P.
\newblock Denoising diffusion probabilistic models.
\newblock In \emph{Advances in Neural Information Processing Systems},
  volume~33, pp.\  6840--6851, 2020.

\bibitem[Ho et~al.(2022)Ho, Chan, Saharia, Whang, Gao, Gritsenko, Kingma,
  Poole, Norouzi, Fleet, et~al.]{ho2022imagen}
Ho, J., Chan, W., Saharia, C., Whang, J., Gao, R., Gritsenko, A., Kingma,
  D.~P., Poole, B., Norouzi, M., Fleet, D.~J., et~al.
\newblock Imagen video: High definition video generation with diffusion models.
\newblock \emph{arXiv preprint arXiv:2210.02303}, 2022.

\bibitem[Ho et~al.(2021)Ho, Chan, Peng, Hang, and Doma{\'n}ski]{ho2021anfic}
Ho, Y.-H., Chan, C.-C., Peng, W.-H., Hang, H.-M., and Doma{\'n}ski, M.
\newblock Anfic: Image compression using augmented normalizing flows.
\newblock \emph{IEEE Open Journal of Circuits and Systems}, 2:\penalty0
  613--626, 2021.

\bibitem[Huang et~al.(2021)Huang, Lim, and Courville]{huang2021variational}
Huang, C.-W., Lim, J.~H., and Courville, A.
\newblock A variational perspective on diffusion-based generative models and
  score matching.
\newblock In \emph{Advances in Neural Information Processing Systems}, 2021.

\bibitem[Hutchinson(1990)]{hutchinson1990stochastic}
Hutchinson, M.~F.
\newblock A stochastic estimator of the trace of the influence matrix for
  laplacian smoothing splines.
\newblock \emph{Communications in Statistics-Simulation and Computation},
  19\penalty0 (2):\penalty0 433--450, 1990.

\bibitem[Karras et~al.(2022)Karras, Aittala, Aila, and
  Laine]{karras2022elucidating}
Karras, T., Aittala, M., Aila, T., and Laine, S.
\newblock Elucidating the design space of diffusion-based generative models.
\newblock In \emph{Advances in Neural Information Processing Systems}, 2022.

\bibitem[Kawar et~al.(2022)Kawar, Elad, Ermon, and Song]{kawar2022denoising}
Kawar, B., Elad, M., Ermon, S., and Song, J.
\newblock Denoising diffusion restoration models.
\newblock In \emph{Advances in Neural Information Processing Systems}, 2022.

\bibitem[Kim et~al.(2022)Kim, Shin, Song, Kang, and Moon]{kim2022soft}
Kim, D., Shin, S., Song, K., Kang, W., and Moon, I.-C.
\newblock Soft truncation: A universal training technique of score-based
  diffusion model for high precision score estimation.
\newblock In \emph{International Conference on Machine Learning}, pp.\
  11201--11228. PMLR, 2022.

\bibitem[Kingma \& Ba(2014)Kingma and Ba]{kingma2014adam}
Kingma, D.~P. and Ba, J.
\newblock Adam: A method for stochastic optimization.
\newblock \emph{arXiv preprint arXiv:1412.6980}, 2014.

\bibitem[Kingma \& Dhariwal(2018)Kingma and Dhariwal]{kingma2018glow}
Kingma, D.~P. and Dhariwal, P.
\newblock Glow: generative flow with invertible 1$\times$ 1 convolutions.
\newblock In \emph{Proceedings of the 32nd International Conference on Neural
  Information Processing Systems}, pp.\  10236--10245, 2018.

\bibitem[Kingma \& Welling(2014)Kingma and Welling]{kingma2013auto}
Kingma, D.~P. and Welling, M.
\newblock Auto-encoding variational bayes.
\newblock In \emph{International Conference on Learning Representations}, 2014.

\bibitem[Kingma et~al.(2021)Kingma, Salimans, Poole, and
  Ho]{kingma2021variational}
Kingma, D.~P., Salimans, T., Poole, B., and Ho, J.
\newblock Variational diffusion models.
\newblock In \emph{Advances in Neural Information Processing Systems}, 2021.

\bibitem[Krizhevsky et~al.(2009)Krizhevsky, Hinton,
  et~al.]{krizhevsky2009learning}
Krizhevsky, A., Hinton, G., et~al.
\newblock Learning multiple layers of features from tiny images.
\newblock 2009.

\bibitem[Lipman et~al.(2022)Lipman, Chen, Ben-Hamu, Nickel, and
  Le]{lipman2022flow}
Lipman, Y., Chen, R.~T., Ben-Hamu, H., Nickel, M., and Le, M.
\newblock Flow matching for generative modeling.
\newblock In \emph{The Eleventh International Conference on Learning
  Representations}, 2022.

\bibitem[Liu et~al.(2022{\natexlab{a}})Liu, Li, Ren, Chen, and
  Zhao]{liu2022diffsinger}
Liu, J., Li, C., Ren, Y., Chen, F., and Zhao, Z.
\newblock Diffsinger: Singing voice synthesis via shallow diffusion mechanism.
\newblock In \emph{Proceedings of the AAAI Conference on Artificial
  Intelligence}, volume~36, pp.\  11020--11028, 2022{\natexlab{a}}.

\bibitem[Liu et~al.(2022{\natexlab{b}})Liu, Gong, et~al.]{liu2022flow}
Liu, X., Gong, C., et~al.
\newblock Flow straight and fast: Learning to generate and transfer data with
  rectified flow.
\newblock In \emph{The Eleventh International Conference on Learning
  Representations}, 2022{\natexlab{b}}.

\bibitem[Loshchilov \& Hutter(2019)Loshchilov and
  Hutter]{loshchilov2017decoupled}
Loshchilov, I. and Hutter, F.
\newblock Decoupled weight decay regularization.
\newblock In \emph{International Conference on Learning Representations}, 2019.

\bibitem[Lu et~al.(2022{\natexlab{a}})Lu, Zheng, Bao, Chen, Li, and
  Zhu]{lu2022maximum}
Lu, C., Zheng, K., Bao, F., Chen, J., Li, C., and Zhu, J.
\newblock Maximum likelihood training for score-based diffusion odes by high
  order denoising score matching.
\newblock In \emph{International Conference on Machine Learning}, pp.\
  14429--14460. PMLR, 2022{\natexlab{a}}.

\bibitem[Lu et~al.(2022{\natexlab{b}})Lu, Zhou, Bao, Chen, Li, and
  Zhu]{lu2022dpm}
Lu, C., Zhou, Y., Bao, F., Chen, J., Li, C., and Zhu, J.
\newblock Dpm-solver: A fast ode solver for diffusion probabilistic model
  sampling in around 10 steps.
\newblock In \emph{Advances in Neural Information Processing Systems},
  2022{\natexlab{b}}.

\bibitem[Meng et~al.(2022)Meng, Song, Song, Wu, Zhu, and Ermon]{meng2021sdedit}
Meng, C., Song, Y., Song, J., Wu, J., Zhu, J.-Y., and Ermon, S.
\newblock {SDE}dit: Image synthesis and editing with stochastic differential
  equations.
\newblock In \emph{International Conference on Learning Representations}, 2022.

\bibitem[Nichol \& Dhariwal(2021)Nichol and Dhariwal]{nichol2021improved}
Nichol, A.~Q. and Dhariwal, P.
\newblock Improved denoising diffusion probabilistic models.
\newblock In \emph{International Conference on Machine Learning}, pp.\
  8162--8171. PMLR, 2021.

\bibitem[Nichol et~al.(2022)Nichol, Dhariwal, Ramesh, Shyam, Mishkin, Mcgrew,
  Sutskever, and Chen]{nichol2021glide}
Nichol, A.~Q., Dhariwal, P., Ramesh, A., Shyam, P., Mishkin, P., Mcgrew, B.,
  Sutskever, I., and Chen, M.
\newblock Glide: Towards photorealistic image generation and editing with
  text-guided diffusion models.
\newblock In \emph{International Conference on Machine Learning}, pp.\
  16784--16804. PMLR, 2022.

\bibitem[Oord et~al.(2016)Oord, Kalchbrenner, Vinyals, Espeholt, Graves, and
  Kavukcuoglu]{van2016conditional}
Oord, A. v.~d., Kalchbrenner, N., Vinyals, O., Espeholt, L., Graves, A., and
  Kavukcuoglu, K.
\newblock Conditional image generation with pixelcnn decoders.
\newblock In \emph{Proceedings of the 30th International Conference on Neural
  Information Processing Systems}, pp.\  4797--4805, 2016.

\bibitem[Ramesh et~al.(2022)Ramesh, Dhariwal, Nichol, Chu, and
  Chen]{ramesh2022hierarchical}
Ramesh, A., Dhariwal, P., Nichol, A., Chu, C., and Chen, M.
\newblock Hierarchical text-conditional image generation with {CLIP} latents.
\newblock \emph{arXiv preprint arXiv:2204.06125}, 2022.

\bibitem[Rombach et~al.(2022)Rombach, Blattmann, Lorenz, Esser, and
  Ommer]{rombach2022high}
Rombach, R., Blattmann, A., Lorenz, D., Esser, P., and Ommer, B.
\newblock High-resolution image synthesis with latent diffusion models.
\newblock In \emph{Proceedings of the IEEE/CVF Conference on Computer Vision
  and Pattern Recognition}, pp.\  10684--10695, 2022.

\bibitem[Salimans \& Ho(2022)Salimans and Ho]{salimans2022progressive}
Salimans, T. and Ho, J.
\newblock Progressive distillation for fast sampling of diffusion models.
\newblock In \emph{International Conference on Learning Representations}, 2022.

\bibitem[Salimans et~al.(2017)Salimans, Karpathy, Chen, and
  Kingma]{salimans2017pixelcnn++}
Salimans, T., Karpathy, A., Chen, X., and Kingma, D.~P.
\newblock Pixelcnn++: Improving the pixelcnn with discretized logistic mixture
  likelihood and other modifications.
\newblock In \emph{International Conference on Learning Representations}, 2017.

\bibitem[Serr{\`a} et~al.(2020)Serr{\`a}, {\'A}lvarez, G{\'o}mez, Slizovskaia,
  N{\'u}{\~n}ez, and Luque]{serra2019input}
Serr{\`a}, J., {\'A}lvarez, D., G{\'o}mez, V., Slizovskaia, O., N{\'u}{\~n}ez,
  J.~F., and Luque, J.
\newblock Input complexity and out-of-distribution detection with
  likelihood-based generative models.
\newblock In \emph{International Conference on Learning Representations}, 2020.

\bibitem[Sohl-Dickstein et~al.(2015)Sohl-Dickstein, Weiss, Maheswaranathan, and
  Ganguli]{sohl2015deep}
Sohl-Dickstein, J., Weiss, E., Maheswaranathan, N., and Ganguli, S.
\newblock Deep unsupervised learning using nonequilibrium thermodynamics.
\newblock In \emph{International Conference on Machine Learning}, pp.\
  2256--2265. PMLR, 2015.

\bibitem[Song et~al.(2021{\natexlab{a}})Song, Meng, and
  Ermon]{song2020denoising}
Song, J., Meng, C., and Ermon, S.
\newblock Denoising diffusion implicit models.
\newblock In \emph{International Conference on Learning Representations},
  2021{\natexlab{a}}.

\bibitem[Song \& Ermon(2019)Song and Ermon]{song2019generative}
Song, Y. and Ermon, S.
\newblock Generative modeling by estimating gradients of the data distribution.
\newblock In \emph{Advances in Neural Information Processing Systems},
  volume~32, pp.\  11895--11907, 2019.

\bibitem[Song et~al.(2020)Song, Garg, Shi, and Ermon]{song2020sliced}
Song, Y., Garg, S., Shi, J., and Ermon, S.
\newblock Sliced score matching: A scalable approach to density and score
  estimation.
\newblock In \emph{Uncertainty in Artificial Intelligence}, pp.\  574--584.
  PMLR, 2020.

\bibitem[Song et~al.(2021{\natexlab{b}})Song, Durkan, Murray, and
  Ermon]{song2021maximum}
Song, Y., Durkan, C., Murray, I., and Ermon, S.
\newblock Maximum likelihood training of score-based diffusion models.
\newblock In \emph{Advances in Neural Information Processing Systems},
  volume~34, pp.\  1415--1428, 2021{\natexlab{b}}.

\bibitem[Song et~al.(2021{\natexlab{c}})Song, Sohl{-}Dickstein, Kingma, Kumar,
  Ermon, and Poole]{song2020score}
Song, Y., Sohl{-}Dickstein, J., Kingma, D.~P., Kumar, A., Ermon, S., and Poole,
  B.
\newblock Score-based generative modeling through stochastic differential
  equations.
\newblock In \emph{International Conference on Learning Representations},
  2021{\natexlab{c}}.

\bibitem[Uria et~al.(2013)Uria, Murray, and Larochelle]{uria2013rnade}
Uria, B., Murray, I., and Larochelle, H.
\newblock {RNADE}: The real-valued neural autoregressive density-estimator.
\newblock \emph{Advances in Neural Information Processing Systems}, 26, 2013.

\bibitem[Vahdat \& Kautz(2020)Vahdat and Kautz]{vahdat2020nvae}
Vahdat, A. and Kautz, J.
\newblock Nvae: a deep hierarchical variational autoencoder.
\newblock In \emph{Proceedings of the 34th International Conference on Neural
  Information Processing Systems}, pp.\  19667--19679, 2020.

\bibitem[Vincent(2011)]{vincent2011connection}
Vincent, P.
\newblock A connection between score matching and denoising autoencoders.
\newblock \emph{Neural computation}, 23\penalty0 (7):\penalty0 1661--1674,
  2011.

\bibitem[Xiao et~al.(2020)Xiao, Yan, and Amit]{xiao2020likelihood}
Xiao, Z., Yan, Q., and Amit, Y.
\newblock Likelihood regret: an out-of-distribution detection score for
  variational auto-encoder.
\newblock In \emph{Proceedings of the 34th International Conference on Neural
  Information Processing Systems}, pp.\  20685--20696, 2020.

\bibitem[Xu et~al.(2022)Xu, Liu, Tegmark, and Jaakkola]{xu2022poisson}
Xu, Y., Liu, Z., Tegmark, M., and Jaakkola, T.~S.
\newblock Poisson flow generative models.
\newblock In \emph{Advances in Neural Information Processing Systems}, 2022.

\bibitem[Xu et~al.(2023)Xu, Liu, Tian, Tong, Tegmark, and
  Jaakkola]{xu2023pfgm++}
Xu, Y., Liu, Z., Tian, Y., Tong, S., Tegmark, M., and Jaakkola, T.
\newblock Pfgm++: Unlocking the potential of physics-inspired generative
  models.
\newblock \emph{arXiv preprint arXiv:2302.04265}, 2023.

\bibitem[Yang \& Mandt(2022)Yang and Mandt]{yang2022lossy}
Yang, R. and Mandt, S.
\newblock Lossy image compression with conditional diffusion models.
\newblock \emph{arXiv preprint arXiv:2209.06950}, 2022.

\bibitem[Zhao et~al.(2022)Zhao, Bao, Li, and Zhu]{zhao2022egsde}
Zhao, M., Bao, F., Li, C., and Zhu, J.
\newblock Egsde: Unpaired image-to-image translation via energy-guided
  stochastic differential equations.
\newblock In \emph{Advances in Neural Information Processing Systems}, 2022.

\end{thebibliography}
\bibliographystyle{icml2023}

\newpage
\appendix
\onecolumn

\section{Different perspectives of diffusion ODEs for bridging the gap between discrete and continuous data}
\label{appendix:perspective}

Suppose the discrete data $\Xv_0$ to be modelled are 8-bit integers $\{0,1,\dots,255\}$. Following the common transform in diffusion models, we first normalize it to range [-1,1] by the mapping $\x_0=\frac{\Xv_0+\frac{1}{2}-128}{128}$. In the following discussions, we consider the model distribution $P_0(\x_0)$ on transformed discrete data $\x_0$, which is equal to $P_0(\Xv_0)$ since the scaling does not alter the discrete probability.

\subsection{Dequantization perspective}
The discrete data $\x_0$ has a uniform gap $\frac{1}{128}$ between two consecutive values on each dimension. We can define the discrete model distribution as
\begin{equation}
    P_0(\x_0)=\int_{\uv\in [-\frac{1}{256},\frac{1}{256}]^d} p_\epsilon(\x+\uv)\dm \uv
\end{equation}
where $p_\epsilon$ is the diffusion ODE defined at time $\epsilon$. Then, we can introduce a dequantization distribution $q(\uv|\x_0)$ with support over $[-\frac{1}{256},\frac{1}{256})^d$. Treating $q$ as an approximate posterior, we obtain the following variational bound~\citep{ho2019flow++}:
\begin{equation}
    \log P_0(\x_0)\geq \E_{q(\uv|\x_0)}\left[\log p_\epsilon(\x_0+\uv)-\log q(\uv|\x_0)\right]
\end{equation}
The ODE term $\log p_\epsilon(\x_0+\uv)$ can be evaluated exactly by solving another ODE called ``Instantaneous Change of Variables"~\citep{chen2018neural}. As for the posterior $\log q(\uv|\x_0)$, we can derive closed-form solutions for predefined posterior formulation. We provide the details for uniform dequantization and our proposed truncated-normal dequantization.
\paragraph{Uniform dequantization} We simply use uniform posterior $q(\uv|\x_0)=\Uc(-\frac{1}{256},\frac{1}{256})$. In this case, $\log q(\uv|\x_0)=d\log 128$ is a constant, and the bound becomes
\begin{equation}
    \log P_0(\x_0)\geq \E_{\uv\sim\Uc(-\frac{1}{256},\frac{1}{256})}\left[\log p_\epsilon(\x_0+\uv)\right]-d\log 128
\end{equation}
Similar to \citet{burda2015importance}, we can also sample multiple $\uv$ to derive a tighter bound, which is called importance weighted estimator:
\begin{equation}
    \log P_0(\x_0)\geq \E_{\uv^{(1)},\dots,\uv^{(K)}\sim\Uc(-\frac{1}{256},\frac{1}{256})}\left[\log \left(\frac{1}{K}\sum_{i=1}^K p_\epsilon(\x_0+\uv^{(i)})\right)\right]-d\log 128
\end{equation}

However, this dequantization will cause a training-evaluation gap. For training, we fit $p_\epsilon$ to the distribution of $\x_\epsilon=\alpha_\epsilon\x_0+\sigma_\epsilon\epsilonv,\epsilonv\sim\Nc(\mathbf 0,\Iv)$. For evaluation, we test $p_\epsilon$ on uniform dequantized $\x_0+\uv,\uv\sim\Uc(-\frac{1}{256},\frac{1}{256})$. This gap will degenerate the likelihood performance, as we will show later.
\paragraph{Truncated-normal dequantization}
To bridge the training-evaluation gap, we test $p_\epsilon$ on $\hat\x_\epsilon = \alpha_\epsilon \x_0 + \sigma_\epsilon \hat\epsilonv$, where $\hat\epsilonv$ obeys a truncated-normal distribution to make sure the range of $\uv$ on each dimension does not exceed $[-\frac{1}{256},\frac{1}{256}]$. Specifically, denote $\tau \coloneqq \frac{\alpha_\epsilon}{256\sigma_\epsilon}$, we define the truncated-normal distribution as
\begin{equation}
    \hat\epsilonv \sim \Tc\Nc\left(\hat\epsilonv\left| \vect{0},\Iv,-\tau,\tau\right.\right)
\end{equation}
Let
\begin{equation}
    \uv \coloneqq \frac{\sigma_\epsilon}{\alpha_\epsilon}\hat\epsilonv \in \left[-\frac{1}{256}, \frac{1}{256}\right]
\end{equation}
By the change of variables for probability density, we have
\begin{equation}
    \log p_\epsilon(\x_0 + \uv)
    = \log p_{\epsilon}\left(
        \x_0 + \frac{\sigma_\epsilon}{\alpha_\epsilon}\hat\epsilonv
    \right)
    = \log p_{\epsilon}(\hat\x_\epsilon) + d\log\alpha_\epsilon 
\end{equation}
\begin{equation}
    \log q(\uv|\x_0)=\log q\left(\frac{\sigma_\epsilon}{\alpha_\epsilon}\hat\epsilonv\right)=\log q(\hat{\epsilonv})+d\log\frac{\alpha_\epsilon}{\sigma_\epsilon}
\end{equation}
where $q(\hat{\epsilonv})$ is the probability distribution function of truncated-normal distributions
\begin{equation}
\label{eqn:tn-pdf}
    q(\hat\epsilonv) = \frac{1}{(2\pi Z^2)^{\frac{d}{2}}}\exp(-\frac{1}{2}\|\hat\epsilonv\|_2^2),\quad Z \coloneqq \Phi(\tau)-\Phi(-\tau)=\text{erf}\left(\frac{\tau}{\sqrt{2}}\right)
\end{equation}
Here $\Phi(\cdot)$ is the cumulative distribution function of standard normal distribution, and $\text{erf}(\cdot)$ is the error function. Combining the equations above, the bound is reduced to
\begin{equation}
\label{eqn:tn_bound_raw}
    \log P_0(\x_0)\geq \E_{q(\hat\epsilonv)}\left[\log p_\epsilon(\hat\x_\epsilon)-\log q(\hat\epsilonv)\right]+d\log\sigma_\epsilon
\end{equation}
Further, we can derive closed-form solutions for the entropy term of truncated-normal distribution:
\begin{equation}
    -\E_{q(\hat\epsilonv)}[\log q(\hat\epsilonv)] = \Hc(q(\hat\epsilonv)) = d\log(\sqrt{2\pi e}) + d\log Z - d\frac{\tau}{\sqrt{2\pi}Z}\exp(-\frac{1}{2}\tau^2)
\end{equation}
and we finally obtain the exact form of the bound:
\begin{equation}
\label{eqn:tn_bound}
    \log P_0(\x_0)
\geq
\E_{\hat\epsilonv\sim \Tc\Nc\left( \vect{0},\Iv,-\tau,\tau\right)}\left[
    \log p_{\epsilon}(\hat\x_\epsilon)
\right] + \frac{d}{2}(1 + \log(2\pi\sigma_\epsilon^2)) + d\log Z - d\frac{\tau}{\sqrt{2\pi}Z}\exp(-\frac{1}{2}\tau^2)
\end{equation}
where the ODE log-likelihood $\log p_{\epsilon}(\hat\x_\epsilon)$ can also be evaluated exactly. Similarly, we have the corresponding importance-weighted estimator by modifying Eqn.~\eqref{eqn:tn_bound_raw}:
\begin{equation}
    \log P_0(\x_0)\geq \E_{\hat\epsilonv^{(1)},\dots,\hat\epsilonv^{(K)}\sim \Tc\Nc\left( \vect{0},\Iv,-\tau,\tau\right)}\left[\log \left(\frac{1}{K}\sum_{i=1}^K\frac{p_\epsilon(\hat\x_\epsilon^{(i)})}{q(\hat\epsilonv^{(i)})}\right) \right]+d\log\sigma_\epsilon
\end{equation}
where $\hat\x_\epsilon^{(i)} \coloneqq \alpha_\epsilon \x_0 + \sigma_\epsilon \hat\epsilonv^{(i)}$, and $q(\hat\epsilonv)$ is expressed in Eqn.~\eqref{eqn:tn-pdf}.

In our experiments, we choose the start time $\gamma_\epsilon=-13.3$. Under this setting, we have $\tau\approx 3$, and the truncated-normal distribution $\Tc\Nc\left( \vect{0},\Iv,-\tau,\tau\right)$ is almost the same as the standard normal distribution $\Nc(\vect{0},\Iv)$ due to the 3-$\sigma$ principle. Thus, $\x_\epsilon$ used in training and $\hat{\x}_\epsilon$ used in testing are virtually identically distributed, resulting in a negligible training-evaluation gap.
\subsection{Variational perspective}
From the variational perspective, we can view the transition from discrete $\x_0$ to continuous $\x_\epsilon$ as a variational autoencoder, where the prior $p_{\epsilon}(\x_{\epsilon})$ is modeled by diffusion ODE, and the approximate posterior $q_{0\epsilon}(\x_{\epsilon}|\x_0)$ is the analytical Gaussian transition kernel in the forward diffusion process at the start. We have the variational bound:
\begin{equation}
\label{eqn:v_bound_raw}
    \log P_0(\x_0)
\geq
\E_{q_{0\epsilon}(\x_{\epsilon}|\x_0)}\left[
    \log p_{\epsilon 0}(\x_0 | \x_{\epsilon})
    + \log p_{\epsilon}(\x_{\epsilon})
    - \log q_{0\epsilon}(\x_{\epsilon}|\x_0)
\right]
\end{equation}
where $q_{0\epsilon}(\x_{\epsilon}|\x_0)=\Nc(\x_\epsilon|\alpha_\epsilon\x_0,\sigma_\epsilon^2\Iv)$. We want to use the reconstruction term $p_{\epsilon 0}(\x_0 | \x_{\epsilon})$ to approximate $q_{\epsilon 0}(\x_0 | \x_{\epsilon})$. Note that
\begin{equation}
    q_{\epsilon0}(\x_0|\x_\epsilon) = \frac{q_{0\epsilon}(\x_\epsilon|\x_0)q_0(\x_0)}{q_\epsilon(\x_\epsilon)}
\end{equation}
for small enough $\epsilon$, we have $q_0(\x_0)\approx q_\epsilon(\x_\epsilon)$, so $q_{\epsilon0}(\x_0|\x_\epsilon) \propto q_{0\epsilon}(\x_\epsilon|\x_0)=\prod_i q_{0\epsilon}(\x_{\epsilon,i}|\x_{0,i})$, where $i$ represents the $i$-th dimension. Thus, we also choose $p_{\epsilon0}(\x_0|\x_\epsilon)$ as a factorized distribution, following \citet{kingma2021variational}:
\begin{equation}
    p_{\epsilon0}(\x_0|\x_\epsilon) = \prod_{i} p_{\epsilon0}(\x_{0,i}|\x_{\epsilon,i})
\end{equation}
where each
\begin{equation}
    p_{\epsilon0}(\x_{0,i}|\x_{\epsilon,i}) \propto q_{0\epsilon}(\x_{\epsilon,i}|\x_{0,i}) \propto \exp\left(-\frac{\left(\x_{\epsilon,i} - \alpha_{\epsilon}\x_{0,i}\right)^2}{2\sigma_\epsilon^2}\right)
\end{equation}
As $\x_0$ is a discrete variable, the probability can be computed by softmax, so we have
\begin{equation}
\label{eqn:variational_recon}
    \log p_{\epsilon0}(\x_0|\x_\epsilon) = \sum_{i=1}^d \log \mathrm{softmax}_{j=0}^{255}\left(-\frac{\left(\x_{\epsilon,i} - \alpha_{\epsilon}j\right)^2}{2\sigma_\epsilon^2}\right)[\x_{0,i}]
\end{equation}
Besides, the Gaussian entropy term can be computed exactly
\begin{equation}
    -\E_{q_{0\epsilon}(\x_\epsilon|\x_0)}[
    \log q_{0\epsilon}(\x_{\epsilon}|\x_0)]=\Hc(q_{0\epsilon}(\x_\epsilon|\x_0))=  \frac{d}{2}(1 + \log(2\pi\sigma_\epsilon^2))
\end{equation}
and the bound is reduced to
\begin{equation}
\label{eqn:variational_bound}
        \log P_0(\x_0)
\geq
\E_{\epsilonv\sim \Nc(\vect0,\Iv)}\left[
\log p_{\epsilon}(\x_{\epsilon})+\log p_{\epsilon 0}(\x_0 | \x_{\epsilon})
\right]+\frac{d}{2}(1 + \log(2\pi\sigma_\epsilon^2))
\end{equation}
where $\x_\epsilon=\alpha_\epsilon\x_0+\sigma_\epsilon\epsilonv$, $\log p_{\epsilon0}(\x_0|\x_\epsilon)$ is given in Eqn.~\eqref{eqn:variational_recon} and $\log p_{\epsilon}(\x_{\epsilon})$ is the exact ODE likelihood. We also have the importance weighted estimator by modifying Eqn.~\eqref{eqn:v_bound_raw}:
\begin{equation}
    \log P_0(\x_0)\geq\E_{\epsilonv^{(1)},\dots,\epsilonv^{(K)}\sim \Nc(\vect0,\Iv)}\left[\log\left(\frac{1}{K}\sum_{i=1}^K\frac{p_\epsilon(\x_\epsilon)p_{\epsilon0}(\x_0|\x_\epsilon)}{q_{0\epsilon}(\x_\epsilon|\x_0)}\right)\right]
\end{equation}
\subsection{Practical connections and results}
Let us consider the bound without importance weighted estimator. By observing the bound in Eqn.~\eqref{eqn:tn_bound} for truncated-normal dequantization and the bound in Eqn.~\eqref{eqn:variational_bound} for variational perspective, we can find that they have similar formulations. Suppose we use $\gamma_\epsilon=-13.3$, we have $\tau\approx3.01869,Z\approx0.9974613$, and the bound in Eqn.~\eqref{eqn:tn_bound} is approximately
\begin{equation}
    \log P_0(\x_0)
\geq
\E_{\hat\epsilonv\sim \Tc\Nc\left( \vect{0},\Iv,-\tau,\tau\right)}\left[
    \log p_{\epsilon}(\hat\x_\epsilon)
\right] + \frac{d}{2}(1 + \log(2\pi\sigma_\epsilon^2))-0.01522\times d
\end{equation}
Next, consider the variational perspective. Though the reconstruction term $\log p_{\epsilon0}(\x_0|\x_\epsilon)$ in Eqn.~\eqref{eqn:variational_bound} depends on the data distribution, empirically it is nearly a constant $\log p_{\epsilon0}(\x_0|\x_\epsilon)\approx -0.01\times d$. So we have the approximate bound
\begin{equation}
            \log P_0(\x_0)
\geq
\E_{\epsilonv\sim \Nc(\vect0,\Iv)}\left[
\log p_{\epsilon}(\x_{\epsilon})
\right]+\frac{d}{2}(1 + \log(2\pi\sigma_\epsilon^2))-0.01\times d
\end{equation}
We note the only difference is that our proposed truncated-normal dequantization uses $\hat\x_\epsilon$ rather than $\x_\epsilon$ for ODE likelihood evaluation, and there is a small constant difference in the bound.

\begin{table}
    \vspace{-.1in}
    \centering
    \caption{\label{tab:different_bound}Likelihood results under different bound and number of importance samples $K$. $K=1$ means we do not use the importance-weighted estimator.}
    \vskip 0.1in
    \begin{small}
    \begin{tabular}{lccccccccc}
    \toprule
    NLL& \multicolumn{3}{c}{Uniform Dequantization} & \multicolumn{3}{c}{Variational}&\multicolumn{3}{c}{Truncated-Normal Dequantization} \\
    \cmidrule{2-10}
    & $K=1$&$K=5$&$K=20$& $K=1$&$K=5$&$K=20$& $K=1$&$K=5$&$K=20$\\
    \midrule
       CIFAR-10 (VP) & 2.74 &2.72  &2.71 &2.60&2.59&2.58&2.60&2.58&2.57\\
       CIFAR-10 (SP) &2.81  & 2.79 &2.78 &2.61&2.59&2.58&2.60&2.57&2.56\\
       ImageNet-32 (VP) &3.52  &3.51  &3.50 &3.46& 3.44 & 3.44 &3.45& 3.44 &3.43\\
       ImageNet-32 (SP) &3.57  & 3.56 &3.55 &3.48&3.47&3.46&3.47&3.45&3.44\\
    \bottomrule
    \end{tabular}
    \end{small}
    \vspace{-0.15in}
\end{table}

\begin{remark}
For high-dimensional data such as images, directly comparing log-likelihood may suffer from scaling issues by the dimension. In practice, we usually compare the BPD (bits/dim) by
\begin{equation}
    \text{BPD}=\E_{\x_0\sim q_0}\left[\frac{-\log P_0(\x_0)}{d\log2}\right]
\end{equation}
where $q_0$ is the data distribution. Since BPD averages the log-likelihood on each dimension, scaling dimensionality has no effect on the final result.
\end{remark}

We test the two types of dequantization and the variational perspective on our final models, using different numbers of importance samples $K$. The results are listed in Table~\ref{tab:different_bound}. Empirically, truncated-normal dequantization performs slightly better than variational, while uniform dequantization gives a bad likelihood due to the large training-evaluation gap. We also observe that increasing $K$ further improves the results by giving a tighter bound.

\begin{figure}[t]
	\centering		
 	\includegraphics[width=.4\linewidth]{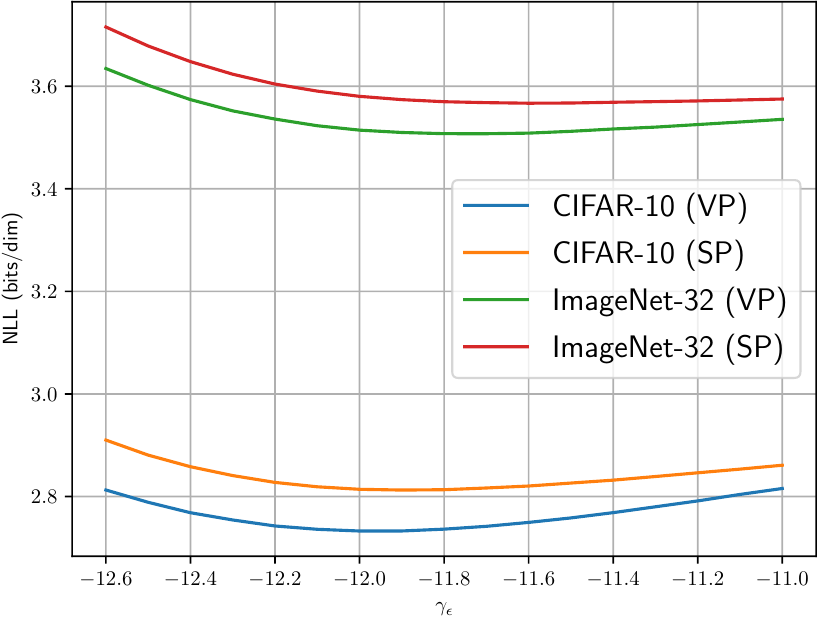}\\
 	\vspace{-.1in}
	\caption{\label{fig:uniform_start}The likelihood evaluation results under uniform dequantization for different start times $\gamma_\epsilon$. To plot the curve, we estimate the likelihood using the first 1024 test samples for CIFAR-10, and the first 512 test samples for ImageNet-32.}
 	\vspace{-.1in}
\end{figure}
\begin{remark}
Since uniform dequantized data has a larger noise level than truncated-normal dequantized data, we find evaluating $\log p_{\epsilon}(\x_0+\uv)$ at start time $\gamma_\epsilon=-13.3$ leads to bad likelihood. Thus, we tune $\gamma_\epsilon$ for uniform dequantization (Figure~\ref{fig:uniform_start}), and eventually choose $\gamma_\epsilon=-12.0,-11.9,-11.7,-11.6$ for CIFAR-10 (VP), CIFAR-10 (SP), ImageNet-32 (VP), ImageNet-32 (SP) respectively.
\end{remark}
\section{Equivalence of different predictors and matching objectives}
\label{appendix:equivalence}
We have the following theorem which demonstrates that different predictors are mutually transformable by a time-dependent skip connection, and they can be trained in a simulation-free approach by equivalent matching objectives.
\begin{theorem}
\label{theorem:equivalence}
Let $\x_0$ be the sample from data distribution, and $\epsilonv$ be the sample from $\Nc(\vect0,\Iv)$. Denote $\x_t=\alpha_t\x_0+\sigma_t\epsilonv, \vv=\dot\alpha_t\x_0+\dot\sigma_t\epsilonv$. Suppose we have four kinds of predictors parameterized by $\theta$ and corresponding matching objectives with positive time weighting function $w(t)$: 
\begin{itemize}
    \item score predictor $\sv_\theta(\x_t,t)$ and score matching loss $\Jc_{\SM}(\theta,w(t))=\E_t\left[w(t)\E_{\x_0,\epsilonv}[\|\s_\theta(\x_t,t)-\nabla_{\x}\log q_t(\x_t)\|_2^2]\right]$
    \item noise predictor $\epsilonv_\theta(\x_t,t)$ and noise matching loss $\Jc_{\NM}(\theta,w(t))=\E_t\left[w(t)\E_{\x_0,\epsilonv}[\|\epsilonv_\theta(\x_t,t)-\epsilonv\|_2^2]\right]$
    \item data predictor $\x_\theta(\x_t,t)$ and data matching loss $\Jc_{\DM}(\theta,w(t))=\E_t\left[w(t)\E_{\x_0,\epsilonv}[\|\x_\theta(\x_t,t)-\x_0\|_2^2]\right]$
    \item velocity predictor $\vv_\theta(\x_t,t)$ and flow matching loss $\Jc_{\FM}(\theta,w(t))=\E_t\left[w(t)\E_{\x_0,\epsilonv}[\|\vv_\theta(\x_t,t)-\vv\|_2^2]\right]$
\end{itemize}
For any $w(t)$, if we denote the optimal (ground-truth) predictors that minimize the corresponding matching losses as $\s^*(\x_t,t),\epsilonv^*(\x_t,t),\x^*(\x_t,t),\vv^*(\x_t,t)$ respectively, then they are equivalent by the following relations:
\begin{equation}
\begin{aligned}
\epsilonv^*(\x_t,t)&=-\sigma_t\s^*(\x_t,t)\\
\x^*(\x_t,t)&=\frac{1}{\alpha_t}\x_t+\frac{\sigma_t^2}{\alpha_t}\s^*(\x_t,t)\\
\vv^*(\x_t,t)&=f(t)\x_t-\frac{1}{2}g^2(t)\s^*(\x_t,t)
\end{aligned}
\end{equation}
where $\s^*(\x_t,t)=\nabla_{\x}\log q_t(\x_t)$ is the ground-truth score.
\end{theorem}
\begin{proof}
For any positive weighting $w(t)$, the overall optimum of the matching loss $\E_t\left[w(t)\E_{\x_0,\epsilonv}[\|\cdot\|_2^2]\right]$ is achieved when the optimum of the inner expectation $\E_{\x_0,\epsilonv}[\|\cdot\|_2^2]$ is achieved for any $t$. For fixed $t$, by denoising score matching~\citep{vincent2011connection}, we know minimizing $\E_{\x_0,\epsilonv}[\|\s_\theta(\x_t,t)-\nabla_{\x}\log q_t(\x_t)\|_2^2]$ is equivalent to minimizing $\E_{\x_0,\epsilonv}[\|\s_\theta(\x_t,t)-\nabla_{\x}\log q_{0t}(\x_t|\x_0)\|_2^2]=\E_{q(\x_t)}\E_{q_{t0}(\x_0|\x_t)}[\|\s_\theta(\x_t,t)-\nabla_{\x}\log q_{0t}(\x_t|\x_0)\|_2^2]$, where $\log q_{0t}(\x_t|\x_0)=-\frac{\epsilonv}{\sigma_t}$. The inner expectation is a minimum mean square error problem, so the optimal score predictor satisfies
\begin{equation}
    \s^*(\x_t,t)=\E_{q_{t0}(\x_0|\x_t)}[\nabla_{\x}\log q_{0t}(\x_t|\x_0)]=-\frac{1}{\sigma_t}\E_{q_{t0}(\x_0|\x_t)}[\epsilonv]
\end{equation}

Similarly, for $\Jc_{\NM}(\theta,w(t))$, the optimal noise predictor satisfies
\begin{equation}
    \epsilonv^*(\x_t,t)=\E_{q_{t0}(\x_0|\x_t)}[\epsilonv]=-\sigma_t\s^*(\x_t,t)
\end{equation}

For $\Jc_{\DM}(\theta,w(t))$, the optimal data predictor satisfies
\begin{equation}
\begin{aligned}
\x^*(\x_t,t)&=\E_{q_{t0}(\x_0|\x_t)}[\x_0]\\
&=\E_{q_{t0}(\x_0|\x_t)}\left[\frac{\x_t-\sigma_t\epsilonv}{\alpha_t}\right]\\
&=\frac{1}{\alpha_t}\x_t-\frac{\sigma_t}{\alpha_t}\E_{q_{t0}(\x_0|\x_t)}[\epsilonv]\\
&=\frac{1}{\alpha_t}\x_t+\frac{\sigma_t^2}{\alpha_t}\s^*(\x_t,t)
\end{aligned}
\end{equation}
For $\Jc_{\FM}(\theta,w(t))$, the optimal velocity predictor satisfies
\begin{equation}
\label{eqn:v_in_expectation}
\begin{aligned}
\vv^*(\x_t,t)&=\E_{q_{t0}(\x_0|\x_t)}[\dot\alpha_t\x_0+\dot\sigma_t\epsilonv]\\
&=\dot\alpha_t\E_{q_{t0}(\x_0|\x_t)}[\x_0]+\dot\sigma_t\E_{q_{t0}(\x_0|\x_t)}[\epsilonv]\\
&=\frac{\dot\alpha_t}{\alpha_t}\x_t+\left(\frac{\dot\alpha_t}{\alpha_t}\sigma_t^2-\sigma_t\dot\sigma_t\right)\s^*(\x_t,t)\\
&=f(t)\x_t-\frac{1}{2}g^2(t)\s^*(\x_t,t)
\end{aligned}
\end{equation}
\end{proof}
The equivalence of optimal predictors also implies the equivalence of parameterized predictors. From the above theorem, we know $\vv_\theta(\x_t,t)$ and $\epsilonv_\theta(\x_t,t)$ are related by $\vv_\theta(\x_t,t)=f(t)\x_t+\frac{g^2(t)}{2\sigma_t}\epsilonv_\theta(\x_t,t)$. In practice, we use $\gamma$ timing. From the relationship $\vv_\theta(\x_t,t)=\vv_\theta(\x_\gamma,\gamma)\frac{\dm\gamma}{\dt},\epsilonv_\theta(\x_t,t)=\epsilonv_\theta(\x_\gamma,\gamma)$, we obtain the noise predictor expressed by $\vv_\theta(\x_\gamma,\gamma)$
\begin{equation}
    \epsilonv_\theta(\x_\gamma,\gamma)=2\frac{\vv_\theta(\x_\gamma,\gamma)-\frac{\dot{\alpha}_\gamma}{\alpha_\gamma}\x_\gamma}{\sigma_\gamma}
\end{equation}
Further, we can replace $\vv_\theta(\x_\gamma,\gamma)$ with the normalized velocity predictor $\tilde\vv_\theta(\x_\gamma,\gamma)=\vv_\theta(\x_\gamma,\gamma)/\sqrt{\dot{\alpha}_\gamma^2+\dot{\sigma}_\gamma^2}$. 

Moreover, we can derive the equivalent training objectives under different parameterizations by employing the relations discussed above freely. For example, when we replace the normalized velocity predictor $\tilde\vv_\theta$ with the score predictor $\sv_\theta$ in the second-order objective Eqn.~\eqref{eqn:second_order_objective}, we can obtain the second-order denoising score matching similar to~\citet{lu2022maximum}. However, though theoretically equivalent, the actual performance of these objectives highly depends on the specific model architecture, hyperparameters and parameterization, and the authors of~\citet{lu2022maximum} find that their high-order denoising score matching objectives only work for VE schedule, but degenerate the performance of pretrained models with VP schedule. 

\section{Specifications under VP and SP schedule}
\label{appendix:specification}
As stated in Section~\ref{sec:practical_consideration}, using $\gamma$ timing and normalized velocity predictor $\tilde\vv_\theta$, the likelihood weighted first-order and second-order flow matching objectives are reformulated as:
\begin{equation}
    \label{eqn:fm_objective}
    \Jc_{\FM}=\int_{\gamma_0}^{\gamma_T}2\frac{\dot{\alpha}_\gamma^2+\dot{\sigma}_\gamma^2}{\sigma_\gamma^2}\E_{\x_0,\epsilonv}\|\tilde\vv_\theta(\x_\gamma,\gamma)-\tilde\vv\|_2^2\dm \gamma
\end{equation}
\begin{equation}
    \Jc_{\FM,\tr}=\int_{\gamma_0}^{\gamma_T}2\frac{\dot{\alpha}_\gamma^2+\dot{\sigma}_\gamma^2}{\sigma_\gamma^2}\E_{\x_0,\epsilonv}\left(\sigma_\gamma\tr(\nabla\tilde\vv_\theta)-\frac{\dot{\sigma}_\gamma}{\sqrt{\dot{\alpha}_\gamma^2+\dot{\sigma}_\gamma^2}}d+\frac{2\sqrt{\dot{\alpha}_\gamma^2+\dot{\sigma}_\gamma^2}}{\sigma_\gamma}\|\tilde\vv_\theta(\x_\gamma,\gamma)-\tilde\vv\|_2^2\right)^2\dm\gamma
\end{equation}

where $\vv=\dot\alpha_\gamma\x_0+\dot\sigma_\gamma\epsilonv, \tilde\vv=\vv/\sqrt{\dot\alpha_\gamma^2+\dot\sigma_\gamma^2}$. For VP and SP schedule, since $\gamma=\log(\sigma_\gamma^2/\alpha_\gamma^2)$, using their schedule properties, $\alpha_\gamma,\sigma_\gamma$ are deterministic functions of $\gamma$ without any hyperparameters. Thus, we can derive their specific objectives and equivalent predictors using the formula for general noise schedules. We summarize them in Table~\ref{tab:specification}, where $\hat{\tilde\vv}_\theta$ denotes the stop-gradient version of $\tilde\vv_\theta$.

\begin{table}[ht]
\label{tab:specification}
    \vspace{-.1in}
    \centering
    \caption{Specification of related values and objectives under VP and SP schedule.}
    \vskip 0.1in
    \resizebox{\textwidth}{!}{%
    \begin{tabular}{lcc}
    \toprule
    Formula&VP&SP\\
    \midrule
        $\alpha_\gamma$&$\displaystyle\sqrt{\frac{1}{1+\exp(\gamma)}}$&$\displaystyle\frac{1}{1+\exp(\gamma/2)}$\\
        $\sigma_\gamma$&$\displaystyle\sqrt{\frac{1}{1+\exp(-\gamma)}}$&$\displaystyle\frac{1}{1+\exp(-\gamma/2)}$\\
        $\dot\alpha_\gamma$&$\displaystyle-\frac{1}{2}\alpha_\gamma\sigma_\gamma^2$&$\displaystyle-\frac{1}{2}\alpha_\gamma\sigma_\gamma$\\
        $\dot\sigma_\gamma$&$\displaystyle\frac{1}{2}\alpha_\gamma^2\sigma_\gamma$&$\displaystyle\frac{1}{2}\alpha_\gamma\sigma_\gamma$\\
        $\sqrt{\dot\alpha_\gamma^2+\dot\sigma_\gamma^2}$&$\displaystyle\frac{1}{2}\alpha_\gamma\sigma_\gamma$&$\displaystyle\frac{1}{\sqrt{2}}\alpha_\gamma\sigma_\gamma$\\
        $\tilde\vv$&$\alpha_\gamma\epsilonv-\sigma_\gamma\x_0$&$\displaystyle\frac{\epsilonv-\x_0}{\sqrt{2}}$\\
        $\Jc_{\FM}$&$\displaystyle\frac{1}{2}\int_{\gamma_0}^{\gamma_T}\alpha_\gamma^2\E_{\x_0,\epsilonv}\|\tilde\vv_\theta(\x_\gamma,\gamma)-\tilde\vv\|_2^2\dm \gamma$&$\displaystyle\int_{\gamma_0}^{\gamma_T}\alpha_\gamma^2\E_{\x_0,\epsilonv}\left\|\tilde\vv_\theta(\x_\gamma,\gamma)-\tilde\vv\right\|_2^2\dm \gamma$\\
       $\Jc_{\FM,\tr}$&$\displaystyle\frac{1}{2}\int_{\gamma_0}^{\gamma_T} \alpha_\gamma^2\E_{\x_0,\epsilonv}\left(\sigma_\gamma\tr(\nabla \tilde\vv_\theta)-\alpha_\gamma d+\alpha_\gamma\|\hat{\tilde\vv}_\theta-\tilde\vv\|_2^2\right)^2\dm\gamma$&$\displaystyle\int_{\gamma_0}^{\gamma_T}\alpha_\gamma^2\E_{\x_0,\epsilonv}\left(\sigma_\gamma\tr(\nabla \tilde\vv_\theta)-\frac{1}{\sqrt{2}} d+\sqrt{2}\alpha_\gamma\|\hat{\tilde\vv}_\theta-\tilde\vv\|_2^2\right)^2\dm\gamma$\\
       $\epsilonv_\theta(\x_\gamma,\gamma)$&$\sigma_\gamma\x_\gamma+\alpha_\gamma \tilde\vv_\theta(\x_\gamma,\gamma)$&$\x_\gamma+\sqrt{2}\alpha_\gamma \tilde\vv_\theta(\x_\gamma,\gamma)$\\
    \bottomrule
    \end{tabular}%
    }
    \vspace{-0.15in}
\end{table}

Next, we derive the designed IS procedure. We want to choose a proposal distribution $p(\gamma)\propto\frac{\dot{\alpha}_\gamma^2+\dot{\sigma}_\gamma^2}{\sigma_\gamma^2}$, which is proportional $\alpha_\gamma^2$ for VP and SP. Since we have explicit expressions for the density, we utilize \textit{inverse transform sampling} to design a sampling procedure. Concretely, we take uniform samples of a number $t\in [0,1]$, and solve the following equation about $\gamma_t$:
\begin{equation}
    \frac{1}{Z}\int_{\gamma_0}^{\gamma_t}\alpha_\gamma^2\dm\gamma=t,\quad Z=\int_{\gamma_0}^{\gamma_1}\alpha_\gamma^2\dm\gamma
\end{equation}
Here we assume maximum time $T=1$, and $Z$ is a normalizing constant.
\paragraph{VP}
We have (omit the constant of the indefinite integral)
\begin{equation}
\int \alpha_\gamma^2\dm\gamma=\log\frac{1}{1+\exp(-\gamma)}=\log\alpha_\gamma^2
\end{equation}
Then the equation for inverse transform sampling is
\begin{equation}
    \log\frac{1}{1+\exp(-\gamma_t)}-\log\alpha_{\gamma_0}^2=Zt,\quad Z=\log\frac{\sigma_{\gamma_1}^2}{\sigma_{\gamma_0}^2}
\end{equation}
The solution has a closed-form expression, which gives the inverse transformation from $t$ to $\gamma$
\begin{equation}
    \gamma_t=\log\frac{1}{\exp(-Zt)/\sigma_{\gamma_0}^2 -1},\quad t\sim \Uc(0,1)
\end{equation}

\paragraph{SP}
We have (omit the constant of the indefinite integral)
\begin{equation}
    \int\alpha_\gamma^2\dm \gamma=-2\left(\log(1+\exp(-\gamma/2))+\frac{1}{1+\exp(-\gamma/2)}\right)
\end{equation}
Denote $F(\gamma)=-\log(1+\exp(-\gamma/2))-(1+\exp(-\gamma/2))^{-1}$, then the equation for inverse transform sampling is
\begin{equation}
    \frac{F(\gamma_t)-F(\gamma_0)}{F(\gamma_1)-F(\gamma_0)}=t
\end{equation}
The solution has no closed-form expressions. Similar to the implementation in \citet{song2021maximum}, we use the bisection method to find the root.

\section{Illustration of velocity prediction and imbalance problem}
\label{appendix:v_pred_interpretation}
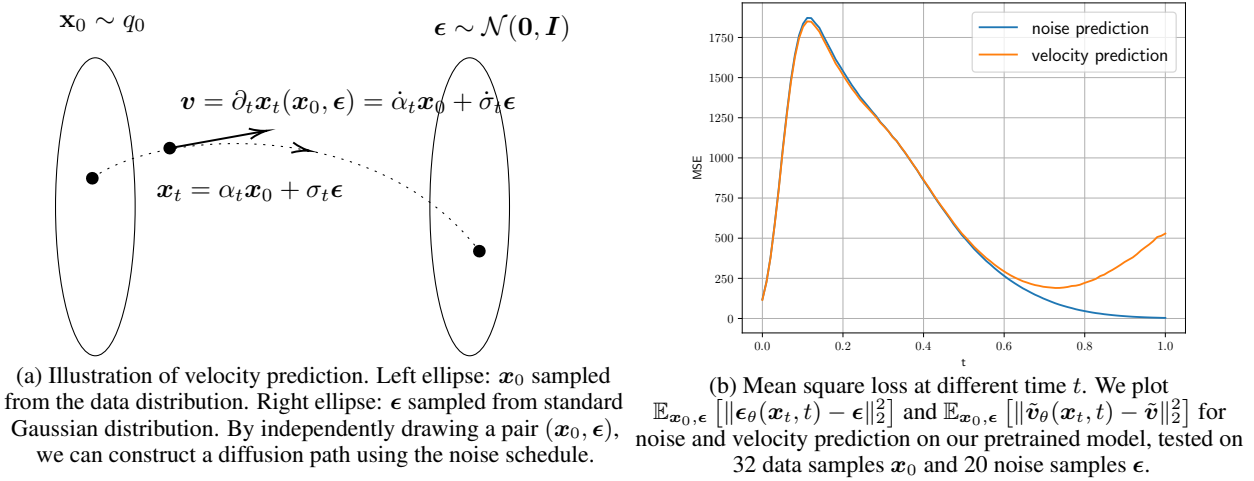
\begin{figure}[ht]
	\centering
	\begin{minipage}{.48\linewidth}
		\centering
			\begin{tikzpicture}[x=0.6pt,y=0.6pt,yscale=-1,xscale=1.0]

\draw   (60,130) .. controls (60,78.09) and (71.24,36) .. (85.1,36) .. controls (98.96,36) and (110.2,78.09) .. (110.2,130) .. controls (110.2,181.91) and (98.96,224) .. (85.1,224) .. controls (71.24,224) and (60,181.91) .. (60,130) -- cycle ;
\draw   (296,130) .. controls (296,78.09) and (307.24,36) .. (321.1,36) .. controls (334.96,36) and (346.2,78.09) .. (346.2,130) .. controls (346.2,181.91) and (334.96,224) .. (321.1,224) .. controls (307.24,224) and (296,181.91) .. (296,130) -- cycle ;
\draw  [dash pattern={on 0.84pt off 2.51pt}]  (83.2,112) .. controls (123.2,82) and (261.2,71) .. (327.2,158) ;
\draw [shift={(327.2,158)}, rotate = 52.82] [color={rgb, 255:red, 0; green, 0; blue, 0 }  ][fill={rgb, 255:red, 0; green, 0; blue, 0 }  ][line width=0.75]      (0, 0) circle [x radius= 3.35, y radius= 3.35]   ;
\draw [shift={(219.98,94.67)}, rotate = 190.55] [color={rgb, 255:red, 0; green, 0; blue, 0 }  ][line width=0.75]    (10.93,-3.29) .. controls (6.95,-1.4) and (3.31,-0.3) .. (0,0) .. controls (3.31,0.3) and (6.95,1.4) .. (10.93,3.29)   ;
\draw [shift={(83.2,112)}, rotate = 323.13] [color={rgb, 255:red, 0; green, 0; blue, 0 }  ][fill={rgb, 255:red, 0; green, 0; blue, 0 }  ][line width=0.75]      (0, 0) circle [x radius= 3.35, y radius= 3.35]   ;
\draw [color={rgb, 255:red, 0; green, 0; blue, 0 }  ,draw opacity=1 ][line width=0.75]    (132,93) -- (191.23,82.35) ;
\draw [shift={(193.2,82)}, rotate = 169.81] [color={rgb, 255:red, 0; green, 0; blue, 0 }  ,draw opacity=1 ][line width=0.75]    (10.93,-3.29) .. controls (6.95,-1.4) and (3.31,-0.3) .. (0,0) .. controls (3.31,0.3) and (6.95,1.4) .. (10.93,3.29)   ;
\draw [shift={(132,93)}, rotate = 349.81] [color={rgb, 255:red, 0; green, 0; blue, 0 }  ,draw opacity=1 ][fill={rgb, 255:red, 0; green, 0; blue, 0 }  ,fill opacity=1 ][line width=0.75]      (0, 0) circle [x radius= 3.35, y radius= 3.35]   ;

\draw (61,7.4) node [anchor=north west][inner sep=0.75pt]    {$\mathbf{x}_{0} \sim q_{0}$};
\draw (297,7.4) node [anchor=north west][inner sep=0.75pt]    {$\epsilonv \sim \N(\mathbf{0},\Iv)$};
\draw (137,54.4) node [anchor=north west][inner sep=0.75pt]    {$\vv=\partial_t \x_t(\x_{0} ,\epsilonv ) =\dot{\alpha }_{t}\x_{0} +\dot{\sigma }_{t}\epsilonv$};
\draw (122,112.4) node [anchor=north west][inner sep=0.75pt]    {$\x_t =\alpha _{t}\x_{0} +\sigma _{t}\epsilonv$};
\end{tikzpicture}\\
\small{(a) Illustration of velocity prediction. Left ellipse: $\x_0$ sampled from the data distribution. Right ellipse: $\epsilonv$ sampled from standard Gaussian distribution. By independently drawing a pair $(\x_0,\epsilonv)$, we can construct a diffusion path using the noise schedule.}
	\end{minipage}
	\begin{minipage}{.48\linewidth}
	\centering
	\includegraphics[width=0.8\linewidth]{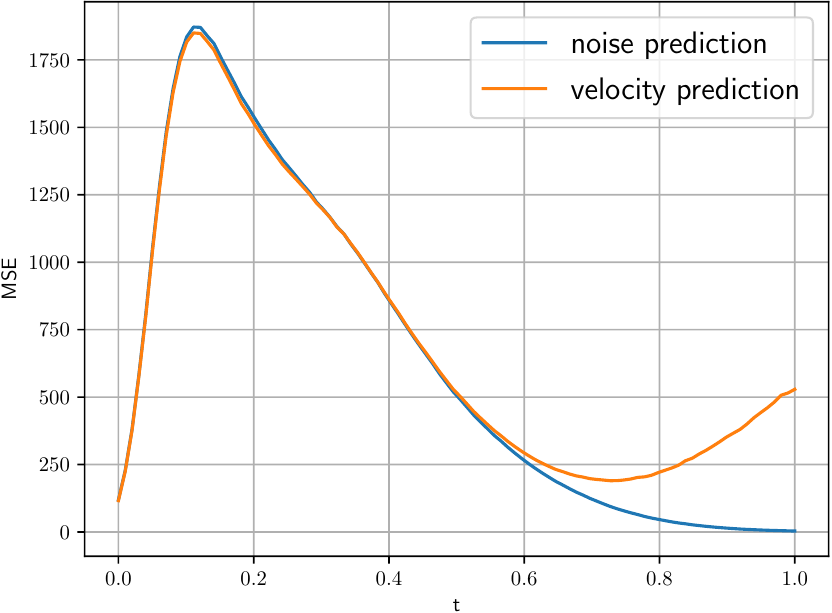}\\
\small{	(b) Mean square loss at different time $t$. We plot $\E_{\x_0,\epsilonv}\left[\|\epsilonv_\theta(\x_t,t)-\epsilonv\|_2^2\right]$ and $\E_{\x_0,\epsilonv}\left[\|\tilde\vv_\theta(\x_t,t)-\tilde\vv\|_2^2\right]$ for noise and velocity prediction on our pretrained model, tested on 32 data samples $\x_0$ and 20 noise samples $\epsilonv$.}
\end{minipage}
	\caption{\label{fig:v_imbalance}Illustration of velocity prediction and imbalance problem.}
\end{figure}

First, we give an intuitive illustration of our velocity parameterization and corresponding flow matching objective in Section~\ref{sec:v_param}. As shown in Figure~\ref{fig:v_imbalance}(a), for each pair $(\x_0, \epsilonv)$ where $\x_0\sim q_0(\x_0)$ and $\epsilonv\sim \N(\vect{0},\Iv)$, let $\x_t = \alpha_t\x_0 + \sigma_t\epsilonv$. As $t$ increases, $\x_t$ moves from $\x_0$ to $\epsilonv$ gradually, forming a diffusion path in the sample space, and $\vv$ is the velocity $\frac{\partial\x_t(\x_0,\epsilonv)}{\partial t}$ across the path. Thus, minimizing $\Jc_{\text{FM}}$ is to predict the expected velocity for all possible $(\x_0,\epsilonv)$ pairs.

Next, we interpret the superiority of velocity prediction from the perspective of balanced prediction difficulty. Intuitively, the noise prediction model suffers from an imbalance problem: at small $t$, $\x_t$ is similar to data, and extracting the insignificant noise component is hard; at large $t$, $\x_t$ is similar to noise, so the noise prediction is easy and has a small error. Velocity prediction, on the other hand, has a property that the prediction target $\vv$ is less relevant to input $\x_t$. In Fig.~\ref{fig:v_imbalance}(b) we empirically confirm it on our pretrained model. We plot the mean square prediction error (MSE) w.r.t. time $t$, which shows that velocity prediction alleviates the imbalance problem by enlarging the training at large $t$. Since the overall error is a weighted combination of the MSE at different $t$ and is invariant to the parameterization,
we can conclude that under noise prediction, the MSE is lower near $t=1$, but is imposed a larger weight, so it has a larger gradient variance.

\section{Relationship between velocity parameterization and other works}
\label{appendix:relationship}

In this section, we demonstrate how the techniques in related works~\citep{karras2022elucidating,lipman2022flow,salimans2022progressive,ho2022imagen} can be reformulated as velocity parameterization.
\subsection{Interpretation by preconditioning}
Works that aim at improving the sample quality of diffusion models also consider the network parameterizations that adaptively mix signal and noise. \citet{karras2022elucidating} proposes to precondition the neural network with a time-dependent skip connection that allows it to estimate either data $\x_0$ or noise $\epsilonv$, or something in between. Similarly, we write the noise predictor $\epsilonv_\theta(\cdot)$ in the following formulation:
\begin{equation}
    \epsilonv_\theta(\x_\gamma,\gamma)=c_{\text{skip}}(\gamma)\x_\gamma+c_{\text{out}}(\gamma)F_\theta(c_{\text{in}}(\gamma)\x_\gamma,\gamma)
\end{equation}
where $F_\theta(\cdot)$ is the pure network, $\x_\gamma=\alpha_\gamma\x_0+\sigma_\gamma\epsilonv$. The flow matching loss can be rewritten as
\begin{equation}
\begin{aligned}
    \Jc_{\FM}(\theta)&=\frac{1}{2}\int_{\gamma_0}^{\gamma_T}\E_{\x_0,\epsilonv}\left[\|\epsilonv_\theta(\x_\gamma,\gamma)-\epsilonv\|_2^2\right]\\
    &=\frac{1}{2}\int_{\gamma_0}^{\gamma_T}\E_{\x_0,\epsilonv}\left[\|c_{\text{skip}}(\gamma)\x_\gamma+c_{\text{out}}(\gamma)F_\theta(c_{\text{in}}(\gamma)\x_\gamma,\gamma)-\epsilonv\|_2^2\right]\\
    &=\frac{1}{2}\int_{\gamma_0}^{\gamma_T}\E_{\x_0,\epsilonv}\left[c_{\text{out}}(\gamma)^2\|F_\theta(c_{\text{in}}(\gamma)\x_\gamma,\gamma)-F_\text{target}(\x_0,\epsilonv,\gamma)\|_2^2\right]
\end{aligned}
\end{equation}
where
\begin{equation}
    F_\text{target}(\x_0,\epsilonv,\gamma)=\frac{\epsilonv-c_{\text{skip}}(\gamma)\x_\gamma}{c_{\text{out}}(\gamma)}
\end{equation}

Following first principles in EDM, We derive formulas for $c_{\text{in}}(\gamma),c_{\text{out}}(\gamma),c_{\text{skip}}(\gamma)$ to ensure:
\begin{enumerate}
    \item The training inputs of $F_\theta(\cdot)$ have unit variance.
    \item The effective training target $F_\text{target}$ has unit variance.
    \item We select $c_{\text{skip}}(\gamma)$ to minimize $c_{\text{out}}(\gamma)$, so that the errors of $F_\theta$ are amplified as little as possible.
\end{enumerate}

From principle 1, we have
\begin{equation}
    \begin{aligned}
    1&=\Var\left[c_{\text{in}}(\gamma)\x_\gamma\right]\\
    1&=\Var\left[c_{\text{in}}(\gamma)(\alpha_\gamma\x_0+\sigma_\gamma\epsilonv)\right]\\
    1&=c_{\text{in}}^2(\gamma)(\alpha_\gamma^2\sigma_{\text{data}}^2+\sigma_\gamma^2)\\
    c_{\text{in}}(\gamma)&=\frac{1}{\sqrt{\sigma_\gamma^2+\sigma_{\text{data}}^2\alpha_\gamma^2}}
    \end{aligned}
\end{equation}

From principle 2, we have
\begin{equation}
    \label{eqn:principle_2}
    \begin{aligned}
    1&=\Var\left[F_\text{target}(\x_0,\epsilonv,\gamma)\right]\\
    1&=\Var\left[\frac{\epsilonv-c_{\text{skip}}(\gamma)\x_\gamma}{c_{\text{out}}(\gamma)}\right]\\
    c_{\text{out}}^2(\gamma)&=\Var\left[\epsilonv-c_{\text{skip}}(\gamma)\x_\gamma\right]\\
    c_{\text{out}}^2(\gamma)&=\Var\left[\epsilonv-c_{\text{skip}}(\gamma)(\alpha_\gamma\x_0+\sigma_\gamma\epsilonv)\right]\\
    c_{\text{out}}^2(\gamma)&=\Var\left[(1-c_{\text{skip}}(\gamma)\sigma_\gamma)\epsilonv-c_{\text{skip}}(\gamma)\alpha_\gamma\x_0\right]\\
    c_{\text{out}}^2(\gamma)&=(1-c_{\text{skip}}(\gamma)\sigma_\gamma)^2+c^2_{\text{skip}}(\gamma)\alpha_\gamma^2\sigma_{\text{data}}^2
    \end{aligned}
\end{equation}

From principle 3, we have
\begin{equation}
\label{eqn:principle_3}
    \begin{aligned}
    0&=\frac{\dm c_{\text{out}}^2(\gamma)}{\dm c_{\text{skip}}(\gamma)}\\
    0&=-2\sigma_\gamma(1-\sigma_\gamma c_{\text{skip}}(\gamma))+2\alpha_\gamma^2\sigma_{\text{data}}^2c_{\text{skip}}(\gamma)\\
    c_{\text{skip}}(\gamma)&=\frac{\sigma_\gamma}{\sigma_\gamma^2+\sigma_{\text{data}}^2\alpha_\gamma^2}
    \end{aligned}
\end{equation}

We now substitute Eqn.~\eqref{eqn:principle_3} into Eqn.~\eqref{eqn:principle_2} to obtain the formula for $c_{\text{out}}(\gamma)$:
\begin{equation}
    c_{\text{out}}(\gamma)=\frac{\sigma_{\text{data}}\alpha_\gamma}{\sqrt{\sigma_\gamma^2+\sigma_{\text{data}}^2\alpha_\gamma^2}}
\end{equation}

If we assume $\sigma_{\text{data}}=1$ and consider VP schedule, we have $\alpha_\gamma^2+\sigma_\gamma^2=1$, and the coefficients are reduced to
\begin{equation}
    c_{\text{in}}(\gamma)=1,\quad c_{\text{skip}}(\gamma)=\sigma_\gamma,\quad c_{\text{out}}(\gamma)=\alpha_\gamma
\end{equation}
In this case, the preconditioning is in agreement with our velocity parameterization by $\tilde{\vv}_\theta(\x_\gamma,\gamma)=F_\theta(\x_\gamma,\gamma)$. In practice, we find setting $\sigma_{\text{data}}=0.5$ as in~\citet{karras2022elucidating} leads to faster descent of the loss at the start, but slower convergence as the training proceeds. 

\subsection{Connection to flow matching in \citet{lipman2022flow}}
\label{appendix:comparison_flow_matching}

\citet{lipman2022flow} defines a conditional probability path $p_t(\x|\x_0)$ that gradually moves the data $\x_0\sim q(\x_0)$ to a target distribution $p_1(\x)$. Note that they use $t=1$ to represent data distribution and $t=0$ to represent target distribution. To be consistent, we reverse their time representation. They obtain the marginal probability path by marginalizing over $q(\x_0)$:
\begin{equation}
    p_t(\x)=\int p_t(\x|\x_0)q(\x_0)\dm \x_0
\end{equation}
They want to learn a vector field $\vv_t(\x)$, which defines a flow $\phiv: [0,1]\times \R^d\rightarrow\R^d$ by
\begin{equation}
    \frac{\dm}{\dt}\phiv_t(\x)=\vv_t(\phiv_t(\x)),\quad \phiv_1(\x)=\x
\end{equation}
so that the marginal $p_t$ can be generated by the push-forward $p_t=[\phiv_t]_* p_1$. In practice, they consider the Gaussian conditional
probability paths
\begin{equation}
    p_t(\x|\x_0)=\Nc(\x|\muv_t(\x_0),\sigma_t^2(\x_0)\Iv)
\end{equation}
and propose a conditional flow matching (CFM) objective for simulation-free training of $\vv_t(\x)$
\begin{equation}
    \Lc_{\CFM}(\theta)=\E_{t,q(\x_0),p_t(\x|\x_0)}\|\vv_t(\x)-\uv_t(\x|\x_0)\|_2^2
\end{equation}
where
\begin{equation}
    \uv_t(\x|\x_0)=\frac{\sigma'_t(\x_0)}{\sigma_t(\x_0)}(\x-\muv_t(\x_0))+\muv'_t(\x_0)
\end{equation}

Suppose the mean $\muv_t(\x_0)$ is linear to $\x_0$, and the standard deviation $\sigma_t(\x_0)$ is invariant to $\x_0$, as the two experimented cases in flow matching. By setting $\muv_t(\x_0)=\alpha_t\x_0,\sigma_t(\x_0)=\sigma_t$, we have
\begin{equation}
    \uv_t(\x|\x_0)=\frac{\dot\sigma_t}{\sigma_t}(\x-\alpha_t\x_0)+\dot\alpha_t\x_0=\dot\alpha_t\x_0+\dot\sigma_t\epsilonv
\end{equation}
where we use $\x=\alpha_t\x_0+\sigma_t\epsilonv,\epsilonv\sim\Nc(\vect0,\Iv)$ since $p_t(\x|\x_0)=\Nc(\x|\alpha_t\x_0,\sigma_t^2\Iv)$. Then we can observe that they are corresponding to our notations: the conditional probability path $p_t(\x|\x_0)$ corresponds to the Gaussian transition kernel $q_{0t}(\x_t|\x_0)$ of the forward diffusion process; the marginal probability path $p_t(\x)$ corresponds to the ground-truth marginals $q_t(\x_t)$ associated with the forward diffusion process; the matching target $\uv_t(\x|\x_0)$ in CFM corresponds to the velocity of the diffusion path $\vv=\dot\alpha_t\x_0+\dot\sigma_t\epsilonv$ in our formulation.

Therefore, the CFM objective in \citet{lipman2022flow} is actually velocity parameterization when specific to Gaussian diffusion processes, which is similar to our first-order objective of the pretraining phase. We can express CFM in a simpler form, which is easier to analyze and generalize to any noise schedule. Then by the equivalence of different predictors (Theorem~\ref{theorem:equivalence}) and the relationship between $f(t),g(t)$ and $\alpha_t,\sigma_t$, we have

\begin{equation}
\begin{aligned}
\Lc_{\CFM}(\theta)&=\int_{0}^{T}\E_{\x_0,\epsilonv}\|\vv_\theta(\x_t,t)-\vv\|_2^2\dm t\\
&=\int_{0}^{T}\E_{\x_0,\epsilonv}\left\|f(t)\x_t-\frac{1}{2}g^2(t)\sv_\theta(\x_t,t)-(\dot\alpha_t\x_0+\dot\sigma_t\epsilonv)\right\|_2^2\dm t\\
&=\int_{0}^{T}\frac{1}{4}g^4(t)\E_{\x_0,\epsilonv}\left\|\sv_\theta(\x_t,t)+\frac{\epsilonv}{\sigma_t}\right\|\dm t
\end{aligned}
\end{equation}
which demonstrates that \textbf{the CFM objective not only changes the parameterization but also imposes a different time weighting $w(t)=\frac{1}{4}g^4(t)$ on the original denoising score matching objective}. When the training aims to improve the sample quality (e.g., FID), the optimal choice for $w(t)$ is still an open problem.

Comparing the CFM objective to our first-order objective Eqn.~\eqref{eqn:fm_objective}, the practical differences are that we use normalized predictor $\tilde{\vv}_\theta$, $\gamma$ timing, and apply likelihood weighting. The likelihood weighting refers to time weighting $w(t)=\frac{g^2(t)}{2\sigma_t^2}$ in Eqn.~\eqref{eqn:sm_raw} and $w(t)=\frac{2}{g^2(t)}$ in Eqn.~\eqref{eqn:fm_1}, which is consistent under different parameterizations and is the theoretically optimal choice for maximum likelihood training~\cite{song2020score}. Also, changing the time domain from $t$ to $\gamma$ will not alter the value of the objective, but will affect the variance of Monte-Carlo estimation and the convergence speed, as we have discussed. For example, the OT path in~\citet{lipman2022flow} is $\alpha_t=1-t,\sigma_t=1-(1-\sigma_{\min})(1-t)\approx t$, and the relation between $\gamma$ and $t$ is $\gamma=\log(\sigma_t^2/\alpha_t^2)=2\log(t/(1-t))$. Under $\gamma$ timing, we can decouple the choice of noise schedules to the greatest extent, and regard the change of variable from $\gamma$ to $t$ as a tunable importance sampling procedure.

Besides, normalizing the field is necessary for stable training of the velocity predictor and is the key to unifying v prediction and preconditioning. Such strategies have also been adopted in more general physics-inspired generative models. For example, \citet{xu2022poisson,xu2023pfgm++} propose to normalize the Poission field when training Poisson flow generative models.

\subsection{Connection to v prediction}

In \citet{salimans2022progressive,ho2022imagen}, a technique called ``v prediction" is used, which parameterizes a network to predict $\mathbf{v}=\alpha_t\epsilonv-\sigma_t\x_0$. Assuming a VP schedule following their choice, we have $\alpha_t^2+\sigma_t^2=1$, so by taking the derivative w.r.t. $t$ we have $\alpha_t\dot\alpha_t+\sigma_t\dot\sigma_t=0$, then
\begin{equation}
\dot\alpha_t=-\frac{\sigma_t\dot\sigma_t}{\alpha_t},\quad \frac{\dm\log \alpha_t}{\dm t}=\frac{\dot\alpha_t}{\alpha_t}=-\frac{\sigma_t\dot\sigma_t}{\alpha_t^2}
\end{equation}
so
\begin{equation}
\begin{aligned}
g^2(t)&=\frac{\dm \sigma_t^2}{\dm t}-2\frac{\dm\log \alpha_t}{\dm t}\sigma_t^2\\
&=2\sigma_t\dot\sigma_t+2\frac{\sigma_t\dot\sigma_t}{\alpha_t^2}\sigma_t^2\\
&=\frac{2\sigma_t\dot\sigma_t}{\alpha_t^2}(\alpha_t^2+\sigma_t^2)\\
&=\frac{2\sigma_t\dot\sigma_t}{\alpha_t^2}
\end{aligned}
\end{equation}
and the velocity is
\begin{equation}
\begin{aligned}
\vv&=\dot{\alpha }_{t}\x_0 +\dot{\sigma }_{t}\epsilonv\\
&=\dot{\sigma }_{t}\epsilonv-\frac{\sigma_t\dot\sigma_t}{\alpha_t}\x_0\\
&=\frac{\dot\sigma_t}{\alpha_t}(\alpha_t\epsilonv-\sigma_t\x_0)\\
&=\frac{\alpha_t}{2\sigma_t}g^2(t)(\alpha_t\epsilonv-\sigma_t\x_0)
\end{aligned}
\end{equation}
Besides, we can compute the normalizing factor as
\begin{equation}
    \sqrt{\dot\alpha_t^2+\dot\sigma_t^2}=\sqrt{\frac{\sigma_t^2\dot\sigma_t^2}{\alpha_t^2}+\dot\sigma_t^2}=\frac{\dot\sigma_t}{\alpha_t}=\frac{\alpha_t}{2\sigma_t}g^2(t)
\end{equation}
so we have the normalized velocity
\begin{equation}
    \tilde{\vv}=\frac{\vv}{\sqrt{\dot\alpha_t^2+\dot\sigma_t^2}}=\alpha_t\epsilonv-\sigma_t\x_0
\end{equation}

Therefore, $\mathbf v=\tilde{\vv}$, which means that v prediction is a special case of velocity parameterization when the noise schedule is VP.
\section{Error-bounded trace of second-order flow matching}
\label{appendix:trace}
Here we provide the proofs for the error-bounded trace of second-order flow matching. First, we provide a lemma that gives the Jacobian of the ground-truth velocity predictor $\vv^*(\x_t,t)$.
\begin{lemma}
Suppose $(\x_0,\x_t)\sim q(\x_0,\x_t)$, denote $\x_t=\alpha_t\x_0+\sigma_t\epsilonv,\vv=\dot\alpha_t\x_0+\dot\sigma_t\epsilonv,\nabla(\cdot)=\nabla_{\x_t}(\cdot)$, we have
\begin{equation}
    \nabla\vv^*(\x_t,t)=\frac{\dot\sigma_t}{\sigma_t}\Iv-\frac{2}{g^2(t)}\E_{q_{t0}(\x_0|\x_t)}\left[(\vv^*(\x_t,t)-\vv)(\vv^*(\x_t,t)-\vv)^\top\right]
\end{equation}
and
\begin{equation}
    \tr(\nabla\vv^*(\x_t,t))=\frac{\dot\sigma_t}{\sigma_t}d-\frac{2}{g^2(t)}\E_{q_{t0}(\x_0|\x_t)}\left[\|\vv^*(\x_t,t)-\vv\|_2^2\right]
\end{equation}
\label{lemma:trace}
\end{lemma}
\begin{proof}
First, the gradient of $q_{t0}$ can be calculated as
\begin{equation}
\begin{aligned}
\nabla q_{t0}(\x_0|\x_t)&=\nabla\frac{q_0(\x_0)q_{0t}(\x_t|\x_0)}{q_t(\x_t)}\\
    &=q_0(\x_0)\frac{q_t(\x_t)\nabla q_{0t}(\x_t|\x_0)-q_{0t}(\x_t|\x_0)\nabla q_t(\x_t)}{q_t(\x_t)^2}\\
    &=\frac{q_0(\x_0)q_{0t}(\x_t|\x_0)}{q_t(\x_t)}\left(\nabla\log q_{0t}(\x_t|\x_0)-\nabla\log q_t(\x_t)\right)\\
    &=q_{t0}(\x_0|\x_t)\left(\nabla\log q_{0t}(\x_t|\x_0)-\nabla\log q_t(\x_t)\right)\\
&=\frac{2}{g^2(t)}(\vv^*(\x_t,t)-\vv)q_{t0}(\x_0|\x_t)
\end{aligned}
\end{equation}
where we use the relation between $\vv^*(\x_t,t)$ and $\nabla\log q_t(\x_t)$ in Theorem~\ref{theorem:equivalence}. From Eqn.~\eqref{eqn:v_in_expectation}, we know $\vv^*(\x_t,t)=\E_{q_{t0}(\x_0|\x_t)}[\vv]$, and for given $\x_0$, we have
\begin{equation}
    \nabla\vv=\nabla\left(\dot\alpha_t\x_0+\dot\sigma_t\frac{\x_t-\alpha_t\x_0}{\sigma_t}\right)=\frac{\dot\sigma_t}{\sigma_t}\Iv
\end{equation}
and
\begin{equation}
    \E_{ q_{t0}(\x_0|\x_t)}\left[(\vv^*(\x_t,t)-\vv)\vv^*(\x_t,t)^\top\right]=\E_{ q_{t0}(\x_0|\x_t)}\left[\vv^*(\x_t,t)-\vv\right]\vv^*(\x_t,t)^\top=\vect 0
\end{equation}
So
\begin{equation}
\begin{aligned}
\nabla\vv^*(\x_t,t)&=\nabla\int q_{t0}(\x_0|\x_t)\vv\dm\x_0 \\
&=\int \nabla q_{t0}(\x_0|\x_t)\vv^\top+q_{t0}(\x_0|\x_t)\nabla\vv\dm\x_0\\
&=\int q_{t0}(\x_0|\x_t)\left(\frac{2}{g^2(t)}(\vv^*(\x_t,t)-\vv)\vv^\top+\frac{\dot\sigma_t}{\sigma_t}\Iv\right)\dm\x_0\\
&=\frac{\dot\sigma_t}{\sigma_t}\Iv+\frac{2}{g^2(t)}\E_{ q_{t0}(\x_0|\x_t)}\left[(\vv^*(\x_t,t)-\vv)\vv^\top\right]\\
&=\frac{\dot\sigma_t}{\sigma_t}\Iv-\frac{2}{g^2(t)}\E_{q_{t0}(\x_0|\x_t)}\left[(\vv^*(\x_t,t)-\vv)(\vv^*(\x_t,t)-\vv)^\top\right]
\end{aligned}
\end{equation}
The expression for $\tr(\nabla\vv^*(\x_t,t))$ can be easily derived from the above equation.
\end{proof}
Then we prove Theorem~\ref{thrm:second_trace} as follows.
\begin{proof}
The optimization in Eqn.~\eqref{eqn:dsm-2-trace-obj} can be rewritten as 
\begin{equation}
    \theta^*\!=\argmin_{\theta}\frac{2\sigma_t^2}{g^2(t)}\E_{q_t(\x_t)}\E_{q_{t0}(\x_0|\x_t)}\!\left[\left|\vv_2^{\text{trace}}(\x_t,t;\theta)\!-\!\frac{\dot{\sigma}_t}{\sigma_t}d\!+\!\frac{2}{g^2(t)}\|\hat\vv_1(\x_t,t)-\vv\|_2^2\right|^2\right].
\end{equation}
For fixed $t$ and $\x_t$, minimizing the inner expectation is a minimum mean square error problem for $\vv_2^{\text{trace}}(\x_t,t;\theta)$, so the optimal $\theta^*$ satisfies
\begin{equation}
    \vv_2^{\text{trace}}(\x_t,t;\theta^*)=\frac{\dot{\sigma}_t}{\sigma_t}d-\frac{2}{g^2(t)}\E_{q_{t0}(\x_0|\x_t)}[\|\hat\vv_1(\x_t,t)-\vv\|_2^2]
\end{equation}
Using Lemma~\ref{lemma:trace} and $\vv^*(\x_t,t)=\E_{q_{t0}(\x_0|\x_t)}[\vv]$, we have
\begin{equation}
\begin{aligned}
&\tr(\nabla\vv^*(\x_t,t))-\vv_2^{\text{trace}}(\x_t,t;\theta^*)\\
=&\frac{2}{g^2(t)}\E_{q_{t0}(\x_0|\x_t)}\left[\|\hat\vv_1(\x_t,t)\|_2^2-2\vv^\top\hat\vv_1(\x_t,t)-\|\vv^*(\x_t,t)\|_2^2+2\vv^\top\vv^*(\x_t,t)\right]\\
=&\frac{2}{g^2(t)}\|\hat\vv_1(\x_t,t)-\vv^*(\x_t,t)\|_2^2
\end{aligned}
\end{equation}
Therefore, we can obtain the error bound by
\begin{equation}
\begin{aligned}
\left|\vv_2^{\text{trace}}(\x_t,t;\theta)-\tr(\nabla_{\x}\vv^*(\x_t,t))\right|&\leq \left|\vv_2^{\text{trace}}(\x_t,t;\theta)-\vv_2^{\text{trace}}(\x_t,t;\theta^*)\right|+|\vv_2^{\text{trace}}(\x_t,t;\theta^*)-\tr(\nabla\vv^*(\x_t,t))|\\
&=\left|\vv_2^{\text{trace}}(\x_t,t;\theta)-\vv_2^{\text{trace}}(\x_t,t;\theta^*)\right|+\frac{2}{g^2(t)}\delta_1^2(\x_t,t)
\end{aligned}
\end{equation}
where $\delta_1(\x_t,t)=\|\hat\vv_1(\x_t,t)-\vv^*(\x_t,t)\|_2$ is the first-order estimation error.
\end{proof}

\section{Difference between our second-order flow matching and the previous time score matching in~\citet{choi2022density}}
We propose the error-bounded second-order flow matching objective to regularize $-\tr(\nabla_{\x}\vv_\theta(\x_t,t))$, which is equal to $\frac{\dm\log p_t(\x_t)}{\dm t}$ by the ``Instantaneous Change of Variables'' formula of CNFs~\citep{chen2018neural}. \citet{choi2022density} proposes a joint score matching method to estimate the data score as well as the time score $(\nabla_{\x}\log p_t(\x), \partial_{t}\log p_t(\x))$, which seems related. However, they are essentially different.

Firstly, the change-of-variable for CNFs describes the total derivative of $\log p_t(\x_t)$ w.r.t. $\x_t$ which evolves by the ODE flow trajectory, not each fixed data point $\x\in \R^d$. However, $\partial_{t} \log p_t(\x)$ in~\citet{choi2022density} describes the partial derivative of $\log p_t(\x)$ for $\x\in \R^d$, i.e., any fixed data point in the whole space. Specifically, according to the Fokker-Planck equation, we have

\begin{equation}
\partial_{t}p_t(\x) = -\nabla_{\x}\cdot (p_{t}(\x) \vv_\theta(\x,t))
\end{equation}

It follows that
\begin{equation}
\partial_{t}\log p_t(\x) = -\nabla_{\x}\cdot \vv_\theta(\x,t) - \vv_\theta(\x,t)^\top \nabla_{\x}\log p_t(\x)
\end{equation}

Therefore, the total derivative $\frac{\dm\log p_t(\x_t)}{\dm t}$ we care about is different from the partial derivative $\partial_{t}\log p_t(x)$ in~\citet{choi2022density}, and their training objectives are also different (with different optimal solutions).

Moreover, there is another difference: \citet{choi2022density} trains another model to estimate the partial derivative $\partial_{t}\log p_t(x)$, which is independent of the ODE velocity $\vv_\theta(\x,t)$ (in the form of the parameterized score function $\sv_\theta(\x,t)$). However, our method restricts the parameterized velocity $\vv_\theta(\x,t)$ itself, and does not employ another model.

Finally, the techniques used in~\citet{choi2022density} and our work are also different. \citet{choi2022density} estimates the score matching loss for the partial derivative $\partial_{t} \log p_t(x)$ by the well-known integral-by-parts, which is used to derive the famous \textit{sliced score matching}~\citep{song2020sliced}, to avoid the computation of the score function $\nabla_{\x}\log p_t(\x)$; However, our method leverages the property of mean square error (that its minimum is conditional mean), which is used to derive the famous \textit{denoising score matching}~\citep{vincent2011connection}, to estimate the divergence of $\vv_\theta(\x,t)$. In the score matching literature, sliced score matching and denoising score matching are two rather different techniques. As first-order denoising score matching is widely used in training diffusion models (such as~\citet{song2020score}), our proposed second-order flow matching is also suitable for training diffusion ODEs.

\section{Details of our adaptive IS}
\label{appendix:is}
In this section, we give details of our adaptive IS stated in Section~\ref{sec:IS}. First, we parameterize $\gamma_\eta(t)$ similar to~\citet{kingma2021variational}:
\begin{equation}
    \label{eqn:gamma_t_func}
    \gamma_\eta(t)=\gamma_0+(\gamma_T-\gamma_0)\frac{\tilde{\gamma}_\eta(t)-\tilde{\gamma}_\eta(0)}{\tilde{\gamma}_\eta(1)-\tilde{\gamma}_\eta(0)}
\end{equation}
where $\tilde{\gamma}_\eta(t)$ is a dense monotone increasing network. Concretely, we use a two-layer fully-connected network $\tilde{\gamma}_\eta(t)=l_2(\phi(l_1(t)))$ where $\phi$ is the sigmoid activation function, $l_1,l_2$ are linear layers with positive weight and output units of 1024 and 1.

\begin{algorithm}[ht!]
\caption{Adaptive importance sampling (single iteration)}
\label{algorithm:adapt_is}
\textbf{Require:} velocity network $\vv_\theta$, IS network $\tilde{\gamma}_\eta$, noise schedule $\alpha_\gamma,\sigma_\gamma$, batch size $N$

\quad Sample $\x_0^{(1)},\dots,\x_0^{(N)}$ from data distribution

\quad Sample $\epsilonv^{(1)},\dots,\epsilonv^{(N)}$ from standard Gaussian distribution $\N(\vect{0},\Iv)$

\quad Sample $t^{(1)},\dots,t^{(N)}$ from uniform distribution $\Uc(0,1)$

\quad Caculate $\gamma_\eta(t^{(i)}),i=1,\dots,N$ by Eqn.~\eqref{eqn:gamma_t_func}

\quad Fix $\eta$, optimize $\theta$ to minimize $\frac{1}{N}\sum_{i=1}^N\mathcal{L}_{\theta,\eta}(\x_0^{(i)},\epsilonv^{(i)},t^{(i)})$

\quad Fix $\theta$, optimize $\eta$ to minimize $\frac{1}{N}\sum_{i=1}^N\mathcal{L}^2_{\theta,\eta}(\x_0^{(i)},\epsilonv^{(i)},t^{(i)})$
\end{algorithm}

Then we present our adaptive IS procedure in Algorithm~\ref{algorithm:adapt_is}. \citet{kingma2021variational} proposes to reuse the gradient $\nabla_\theta\mathcal{L}_{\theta,\eta}(\x_0,\epsilonv,t)$ to optimize $\eta$ and avoid a second backpropagation by decomposing the gradient $\nabla_\eta \mathcal{L}_{\theta,\eta}^2(\x_0,\epsilonv,t)$ using chain-rule. We simply their learning of $\tilde{\gamma}_\eta$ by removing the complex gradient operation in one iteration and propose to alternatively optimize $\theta$ and $\eta$. It may take extra overhead, but also seeks the optimal IS and is enough for ablation.

\section{Experiment details}
\label{appendix:details}
In this section, we provide details of our experiment settings. Our network, hyperparameters and training are the same for different noise schedules on the same dataset.
\paragraph{Model architectures} Our diffusion ODEs are parameterized in terms of the $\gamma$-timed normalized velocity predictor $\tilde\vv_\theta(\x_\gamma,\gamma)$, based on the U-Net structure of \citet{kingma2021variational}. This architecture is tailored for maximum likelihood training, employing special designs such as removing the internal downsampling/upsampling and adding Fourier features for fine-scale prediction. Our configuration for each dataset also follows \citet{kingma2021variational}: For CIFAR-10, we use U-Net of depth 32 with 128 channels; for ImageNet-32, we still use U-Net of depth 32, but double the number of channels to 256. All our models use a dropout rate of 0.1 in the intermediate layers. For CIFAR-10 (with data augmentation), we use U-Net of depth 32 with 256 channels and decrease the dropout rate to 0.05.
\paragraph{Hyperparameters and training}
We follow the same default training settings as \citet{kingma2021variational}. For all our experiments, we use the Adam~\citep{kingma2014adam} optimizer with learning rate $2\times 10^{-4}$, exponential decay rates of $\beta_1=0.9,\beta_2=0.99$ and decoupled weight decay~\citep{loshchilov2017decoupled} coefficient of 0.01. We also maintain an exponential moving average (EMA) of model parameters with an EMA rate of 0.9999 for evaluation. 

For other hyperparameters, we use fixed start and end times which satisfy $\gamma_\epsilon=-13.3,\gamma_T=5.0$, which is the default setting in \citet{kingma2021variational}. In the finetuning stage, we simply set the coefficient $\lambda$ in the mixed loss $\Jc_{\FM}(\theta)+\lambda\Jc_{\FM,\tr}(\theta)$ as 0.1 with no further tuning, so that the magnitude of the second-order loss is negligible w.r.t the first-order loss. Since the first-order matching accuracy is critical to the second-order matching, a large $\lambda$ will make the training unstable or even degenerate the likelihood performance.

All our training processes are conducted on 8 GPU cards of NVIDIA A40 except for ImageNet-32 (old version) and CIFAR-10 (with data augmentation). For CIFAR-10, we pretrain the model for 6 million iterations, which takes around 3 weeks. Then we finetune the model for 200k iterations, which takes around 1 day. For ImageNet-32 (new version), we pretrain the model for 2 million iterations, which takes around 2 weeks. Then we finetune the model for 250k iterations, which takes around 3 days. We use a batch size of 128 for both training stages and both datasets.

Note that in related works~\citep{lipman2022flow,albergo2022building}, experiments on ImageNet-32 (new version) are conducted at a larger batch size (512 or 1024), which may improve the results. We did not use a larger batch size or train longer due to resource limitations.

For ImageNet-32 (old version), the training processes are conducted on 8 GPU cards of NVIDIA A100 (40GB). We pretrain the model for 2 million iterations using a batch size of 512, which takes around 2 weeks. Then we finetune the model for 500k iterations using a batch size of 128 and accumulate the gradient for every 4 batches, which takes around 2.5 days.

For CIFAR-10 (with data augmentation), the training processes are conducted on a cluster of 64 GPU cards of NVIDIA A800 (80GB). We pretrain the model for 2 million iterations using a batch size of 1024, which takes around 2 weeks. Due to the large training resource requirements and the regularization effect by data augmentation, we do not further finetune the model by the second-order flow matching loss.
\paragraph{Likelihood and sample quality} For likelihood, we use our truncated-normal dequantization. When the number of importance samples $K=1$, we report the BPD on the test dataset with 5 times repeating to reduce the variance of the trace estimator. When $K=5$ or $K=20$, we do not repeat the dataset since the log-likelihood of a data sample is already evaluated multiple times. For sampling, since we are concentrated on ODE, we simply use an adaptive-step ODE solver with RK45 method~\citep{dormand1980family} (relative tolerance $10^{-5}$ and absolute tolerance $10^{-5}$). We generate 50k samples and report the FIDs on them. Utilizing high-quality sampling procedures such as PC sampler~\citep{song2020score} or fast sampling algorithms such as DPM-Solver~\citep{lu2022dpm} may improve the results, which are left for future work.

\section{Additional samples}
\label{appendix:samples}
\begin{figure}[ht]
	\centering
	\begin{minipage}{.28\linewidth}
		\centering
			\includegraphics[width=.96\linewidth]{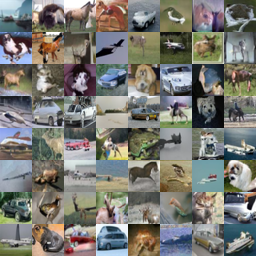}\\
\small{(a) VDM\citep{kingma2021variational}}
	\end{minipage}
	\begin{minipage}{.28\linewidth}
	\centering
	\includegraphics[width=.96\linewidth]{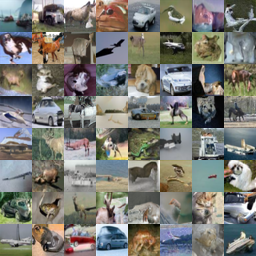}\\
\small{	(b) Our pretrain}
\end{minipage}
	\begin{minipage}{.28\linewidth}
		\centering		
 	\includegraphics[width=.96\linewidth]{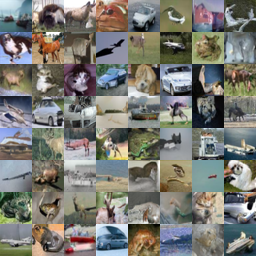}\\
\small{(c) Our pretrain+finetune}
   \end{minipage}
	\caption{Random samples for ablation by ODE sampler. Our pretraining and finetuning lead to a better likelihood with small visual quality degeneration.}
\end{figure}

\begin{figure}[ht]
	\centering		
 	\includegraphics[width=.5\linewidth]{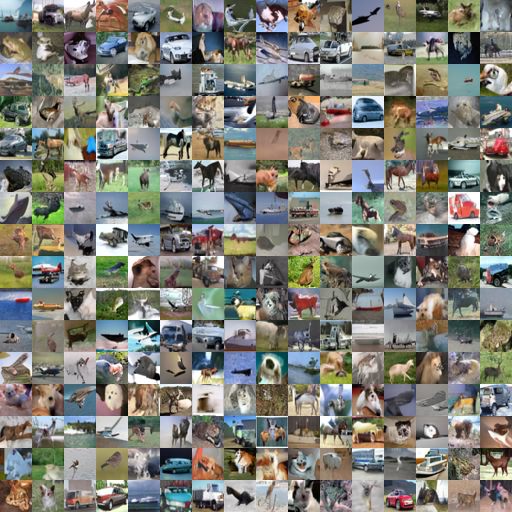}\\
 	\vspace{-.1in}
	\caption{Random samples by ODE sampler (CIFAR-10, SP).}
 	\vspace{-.1in}
\end{figure}

\begin{figure}[ht]
	\centering		
 	\includegraphics[width=.6\linewidth]{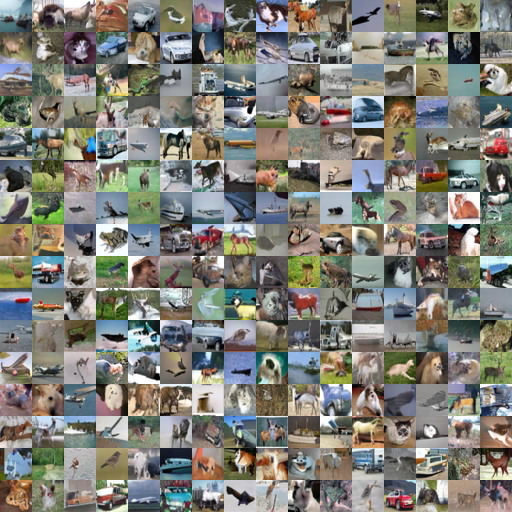}\\
 	\vspace{-.1in}
	\caption{Random samples by ODE sampler (CIFAR-10, VP).}
 	\vspace{-.1in}
\end{figure}

\begin{figure}[ht]
	\centering		
 	\includegraphics[width=.6\linewidth]{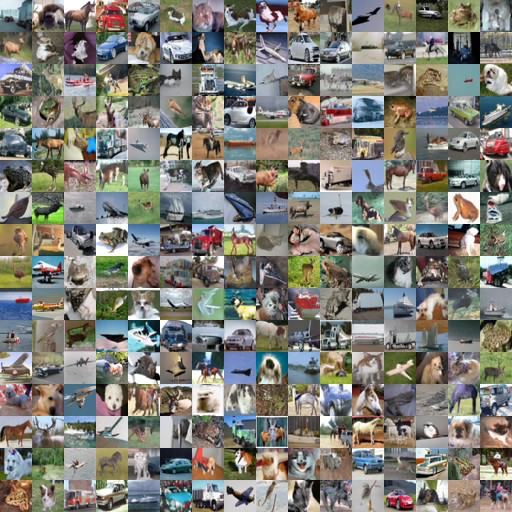}\\
 	\vspace{-.1in}
	\caption{Random samples by ODE sampler (CIFAR-10, VP, with augmentation).}
 	\vspace{-.1in}
\end{figure}

\begin{figure}[ht]
	\centering		
 	\includegraphics[width=.6\linewidth]{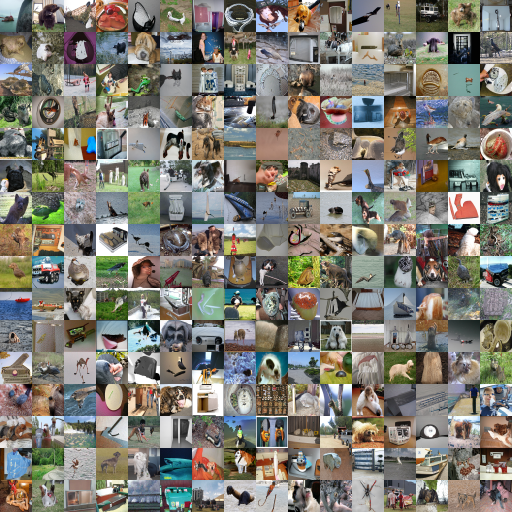}\\
 	\vspace{-.1in}
	\caption{Random samples by ODE sampler (ImageNet-32, VP).}
 	\vspace{-.1in}
\end{figure}
\begin{figure}[ht]
	\centering		
 	\includegraphics[width=.6\linewidth]{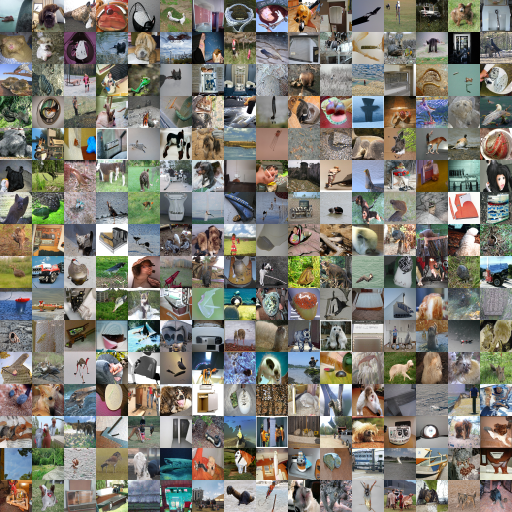}\\
 	\vspace{-.1in}
	\caption{Random samples by ODE sampler (ImageNet-32, SP).}
 	\vspace{-.1in}
\end{figure}


\end{document}